\documentclass{aims}
\usepackage{amssymb,amsbsy,amsmath,amsfonts,amssymb,amscd}
  \usepackage{paralist}
  \usepackage{lipsum}
  \usepackage[pagewise]{lineno}
  \usepackage{graphics} 
  \usepackage{epsfig} 
  \usepackage{algorithm,algpseudocode}
  \usepackage{enumitem}
\usepackage{graphicx}  \usepackage{epstopdf}
 \usepackage[colorlinks=true]{hyperref}
\hypersetup{urlcolor=blue, citecolor=red}

\def\currentvolume{X}
 \def\currentissue{X}
  \def\currentyear{200X}
   \def\currentmonth{XX}
    \def\ppages{X--XX}

\newcommand{\EE}{\mathbb{E}}

\newcommand{\rd}{\,\mathrm{d}}

\renewcommand{\star}{\ast}
\newtheorem{theorem}{Theorem}[section]
\newtheorem{corollary}{Corollary}

\newtheorem{lemma}[theorem]{Lemma}
\newtheorem{proposition}{Proposition}

\theoremstyle{definition}

\newtheorem{remark}{Remark}

\title[Constrained Ensemble Langevin Monte Carlo] 
      {Constrained Ensemble Langevin Monte Carlo}

\author[Zhiyan Ding and Qin Li]{}


 \email{zding49@math.wisc.edu}
 \email{qinli@math.wisc.edu}

\thanks{Q.L. acknowledges support from Vilas Early Career award. The research of Z.D., and Q.L is supported in part by NSF via grant DMS-1750488, DMS-2023239 and Office of the Vice Chancellor for Research and Graduate Education at the University of Wisconsin Madison with funding from the Wisconsin Alumni Research Foundation.}

\thanks{$^*$ Corresponding author: Zhiyan Ding}

\begin{document}
\maketitle

\centerline{\scshape Zhiyan Ding$^*$}
\medskip
{\footnotesize
 \centerline{Department of Mathematics}
   \centerline{University of Wisconsin-Madison}
   \centerline{Madison, WI 53705 USA}
} 

\medskip

\centerline{\scshape Qin Li}
\medskip
{\footnotesize
 \centerline{Department of Mathematics}
   \centerline{University of Wisconsin-Madison}
   \centerline{Madison, WI 53705 USA}
} 

\bigskip

 \centerline{(Communicated by the associate editor name)}

\begin{abstract}
The classical Langevin Monte Carlo method looks for samples from a target distribution by descending the samples along the gradient of the target distribution. The method enjoys a fast convergence rate. However, the numerical cost is sometimes high because each iteration requires the computation of a gradient. One approach to eliminate the gradient computation is to employ the concept of ``ensemble." A large number of particles are evolved together so the neighboring particles provide gradient information to each other. In this article, we discuss two algorithms that integrate the ensemble feature into LMC, and the associated properties.

In particular, we find that if one directly surrogates the gradient using the ensemble approximation, the algorithm, termed Ensemble Langevin Monte Carlo, is unstable due to a high variance term. If the gradients are replaced by the ensemble approximations only in a constrained manner, to protect from the unstable points, the algorithm, termed Constrained Ensemble Langevin Monte Carlo, resembles the classical LMC up to an ensemble error but removes most of the gradient computation.
\end{abstract}

\section{Introduction}
Bayesian sampling is one of the core problems in Bayesian inference. It has a wide applications in data assimilation and inverse problems~\cite{Reich2011,MCMCforML} that arise in remote sensing and imaging~\cite{LiNewton_MCMC}, atmospheric science and earth science~\cite{FABIAN198117}, petroleum engineering~\cite{OGinverse,PES} and epidemiology~\cite{COVID_travel}. The goal is to find i.i.d. samples or approximately i.i.d. samples from a probability distribution that encodes the information of an unknown parameter. Throughout the paper we denote
\begin{equation}
p(x)\propto e^{-f(x)}\,,\quad x\in\mathbb{R}^d
\end{equation}
the distribution function of the unknown parameter $x$, and we assume that $\nabla f(x)$ is $L$-smooth, meaning $\nabla f$ is Lipschitz continuous with $L$ being its Lipschitz constant: $|\nabla f(y) - \nabla f(x)|<L|x-y|$.

There are many successful sampling algorithms~\cite{Neal2001,BESKOS20171417,Doucet2001,Neal1993}. One class of classical sampling approach is the celebrated Markov chain Monte Carlo (MCMC)~\cite{Neal1993,Robert2004,MCSH,DUANE1987216,Geman1984}. This is a class of methods that sets the target distribution as the invariant measure of the Markov transition kernel, so after many rounds of iteration, the sample can be viewed to be drawn from the invariant measure. Since there are many ways to design the Markov chain, there are many subcategories of MCMC methods. Among them, the Langevin Monte Carlo (LMC) stands out for its simplicity, and fast convergence rate.

The key idea of LMC is to design a stochastic differential equation, whose long time equilibrium coincides with the target distribution. The samples are then drawn by following the trajectory of the (discretized) SDE. Typically the SDE converges exponentially fast, and thus the probability distribution of LMC samples, viewed as the discrete version of the SDE, also converges to the target distribution exponentially fast, up to a discretization error. The non-asymptotic convergence rate for these methods and their variations was recently made rigorous in~\cite{doi:10.1111/rssb.12183,DALALYAN20195278,durmus2017,durmus2018analysis,JMLR:v20:19-306,doi:10.1137/19M1284014} for log-concave probability distribution functions (or equivalently, for convex $f(x)$).

One key drawback of LMC is that it requires the frequent calculation of the gradients. For each sample, at each iteration, one needs to compute at least one full gradient. For a problem in $\mathbb{R}^d$, this is a calculation of $d$ partial derivatives per sample per iteration, and in the case when $d\gg 1$, the cost is rather high. Therefore, in the most practical setting, one looks for substitutes of LMC that achieve ``gradient-free" property so that the number of partial derivative computation is relaxed~\cite{NIPS2020_2365,doi:10.1137/19M1284014}.

Another sampling strategy that is completely parallel to the MCMC method is the ensemble type method. Unlike MCMC, or LMC in particular, ensemble methods evolve a large number of samples altogether, and these samples interplay with each other. A Fokker-Planck type PDE is formulated to drive an arbitrarily given distribution toward the target distribution, and the ensemble methods can be viewed as the particle methods applied to numerically evolve the PDE, with the ensemble distribution of the samples approximating the solution of the PDE. Two famous ensemble methods are Ensemble Kalman Inversion~\cite{Iglesias_2013,SS} and Ensemble Kalman Sampling~\cite{EKS,nusken2019note,doi:10.1137/19M1304891}. Earlier works are found in~\cite{Reich2011,Evensen:2006:DAE:1206873,Mattewarticle}. See also the numerical analysis and other follow up works in~\cite{ding2019ensemble,ding2019ensemble2,Herty_2020,zhang2021langevinized}.

The main drawbacks of ensemble methods are also obvious: The algorithms surrogate the statistical quantities with the ensemble version, introducing new computational cost and some ensemble error. Numerical analysis essentially needs to trace the propagation of such ensemble error, and is typically very involved. There is, however, one factor of ensemble methods that can potentially bring a great benefit: Since a lot of samples are evolved together on $\mathbb{R}^d$, it is easy to imagine that close neighbors of each sample can already approximately provide the gradient information. This may make gradient-free computation possible. Indeed, suppose one has a large number of particles, sampled from a certain probability distribution, in a small neighborhood of a sample $x^\ast$, then taking the average of the finite differences between these particles can give a rather good estimate to the gradient $\nabla f(x^\ast)$ to be used in LMC. This idea was already explored in EKS, where the authors inserted a variance term in the underlying SDE of LMC, and by combining the gradient term with the variance term, they formed a covariance that requires no gradient computation. However, such strategy holds true either if the forward map is linear, or the samples are all controllably close to each other. It is hard to justify either in real practice. Nevertheless, such exploration sets a stepping stone for designing gradient-free methods under the ensemble framework.

To summarize, the non-asymptotic convergence rate of LMC is thoroughly studied for a large class of nonlinear $f(x)$, while the validity of ensemble methods are generally lacking. On the other hand, LMC requires the computation of gradients, but the strategy of evolving a large number of samples as is done in the ensemble methods can potentially eliminate the gradient computation.

It is thus natural to ask if it is possible to bring together the two approaches for a new method that may inherit the advantages of both. To be specific, we look for an algorithm that requires as few gradient calculations as possible, while being able to sample (almost) exponentially fast in time. One attempt of breeding the two methods was taken in~\cite{zhang2021langevinized} where the authors added another layer of LMC into EnKF and designed the so-called Langevined EnKF. For linear $f(x)$ they can show the consistency, and in the nonlinear case, gradients are nevertheless needed. Therefore 
the advantage of removing the gradient computation using the concept of ensemble is lost. We look for the possibility of replacing gradients using the neighbor information whenever possible, and have a very different goal in this paper.

As such, we provide two sides of the answer:
\begin{itemize}
\item We first study the most straightforward approach. This is to sample a large number of particles altogether and in each iteration for the updates, we replace every gradient in LMC by the ensemble approximation. We term this method Ensemble LMC (EnLMC). This algorithm, despite being intuitive, will be shown to be unstable. Indeed, at the ``outskirts" of $p(x)$, the accuracy of the updates very sensitively depend on the gradient, and the error induced by the surrogate can be significantly enlarged. This instability suggests that the replacement should not be enacted in these regions.
\item We therefore propose an alternative, termed Constrained Ensemble LMC (CEnLMC). The constrained version of EnLMC enacts the ensemble approximation to the gradient only in the stable region, and for samples in the unstable region, we directly compute $\nabla f$. We can show that this method provides samples that are close to LMC samples, and thus converges to the target distribution at the same rate (exponential, up to a controllable error term). Furthermore, we present how the parameters in the constraints determine the stability of the algorithm and the chance of enacting ensemble approximations.
\end{itemize}

We stress that the method CEnLMC is not completely ``gradient-free" since it enacts ensemble approximation to replace the gradient computation only in the ``stable" regions. However, the study conducted here presents an understanding on how to fuse the concepts of ensemble methods and LMC. While the new method provides a possibility to reduce the gradient computation, it also embraces the fast convergence that can be achieved by LMC for nonlinear $f$.

We also mention that there are many means for approximating the gradients. We cannot claim the optimality of the ensemble approximation used in this article. It is highly possible that one can replace the gradients in LMC using other methods that explore information from neighboring ensemble samples in a more efficient way (see Appendix~\ref{sec:ega_FD} for a negative example). This line of research requires a more detailed study on multiple choices of ensemble approximation and is beyond the scope of the current paper. The current result is one of the pioneering attempts to integrate ensemble features to LMC, and shed light on inventing algorithms that both converge fast and are gradient-free.

~\textcolor{black}{Lastly, we mention that in some communities (optimization for example), the algorithms that avoid or use gradients are termed zero-th order and first order methods. Similarly methods that use hessian information are of second order. The method we propose in this article can be viewed in between zero-th and first, since it eliminates a large portion of gradient calculations. Compared to zero-th order method, the advantages are obvious. All zero-th order methods converge slowly. One such example is the random walk Metropolis (RWM) that converges in $O(d^2)$ iterations~\cite{JMLR:v20:19-306}. On the contrary, LMC converges in $O(d)$~\cite{DALALYAN20195278}, or sometimes $O(d^{1/2})$ iterations when $f$ is sufficiently smooth~\cite{li2021sqrtd}. Our method matches the convergence rate as the classical LMC, but eliminates gradients, meaning it achieves the first order convergence with a zero-th order cost.}

\textcolor{black}{The paper is organized as follows. In Section~\ref{sec:review}, we review two main ingredients of our methods: the classical LMC, and the ensemble gradient approximation. In Section~\ref{sec:alg}, we propose the two new methods and discuss the properties. More specifically, we will show the brute-force combination of LMC and the ensemble gradient approximation will lead to an unstable algorithm (EnLMC), but the constrained version (CEnLMC) recovers the target distribution with a high numerical saving. We show two numerical examples to demonstrate the saving and the accuracy in Section~\ref{sec:numerics}. The proof is given in Section~\ref{sec:proof}.}

\section{Two main ingredients}\label{sec:review}
The main ingredients of our method are the classical Langevin Monte Carlo and an ensemble approximation to the gradient. We review them in this section.

\subsection{Langevin Monte Carlo (LMC)}
LMC is a very popular MCMC type sampling method. Under mild conditions, it provides fast convergence: after a few rounds of iterations, samples can be viewed approximately drawn from the target distribution.

The classical LMC starts with a sample, denoted as $x^0$, and updates the sample position according to:
\begin{equation}\label{eqn:LMCdiscretization}
x^{m+1}=x^m-\nabla f(x^m)h+\sqrt{2h}\xi^m_d\,,
\end{equation}
where $h$ is the time stepsize, and $\xi^{m}_d$ is drawn i.i.d. from $\mathcal{N}(0,I_d)$, and $I_d$ denotes the identity matrix of size $d\times d$. For a fixed small $h$, as $m\to\infty$, it is expected that $q^m$, the probability distribution of $x^m$, gets close to $p$, the target distribution.

To intuitively understand the convergence of this algorithm, we can view the updating formula as the Euler-Maruyama discretization for the following SDE:
\begin{equation}\label{eqn:SDE}
\rd X_t=-\nabla f(X_t)\rd t+\sqrt{2}\rd B_t\,,
\end{equation}
where $B_t$ is a $d$-dimensional Brownian motion. The SDE characterizes the trajectory of $X_t$ by the forcing term $\nabla f(X)\rd t$ and the random walk $\rd B_t$. While $\nabla f$ drives $X_t$ to the minimum of $f$, the Brownian motion term introduces the fluctuation. Denote $q^0(x)$ the initial distribution from where $X_0$ is drawn, and $q(x,t)$ the probability density function of $X_t$, then it is a well-known result that $q(x,t)$ satisfies the following Fokker-Planck equation:
\begin{equation}\label{eqn:FKPKLangevin}
\partial_t q=\nabla\cdot(\nabla fq+\nabla q)\,,\quad\text{with}\quad q(x,0) = q^0\,.
\end{equation}
It was shown in~\cite{Markowich99onthe} that $q(x,t)$ converges to the target density function $p(x) \propto e^{-f}$ exponentially fast in time, meaning:
\[
\lim_{t\rightarrow\infty}X_t\sim p(x)\,.
\]

Considering that the updating formula for LMC~\eqref{eqn:LMCdiscretization} is merely a discretization of~\eqref{eqn:SDE}, then $x^m\approx X_{mh}$, and thus for large enough $m$, $q^m$, the distribution of $x^m$, should also be close to $p$. This is made rigorous recently in a number of papers~\cite{doi:10.1111/rssb.12183,DALALYAN20195278,durmus2017,durmus2018analysis}, most of which quantize the difference between $q^m$ and $p$ using the Wasserstein distance. To be more specific, it was shown in~\cite{DALALYAN20195278,durmus2018analysis} that for strongly-convex, gradient-Lipschitz $f$, to achieve $\epsilon$ accuracy in Wasserstein $L_2$ distance, the number of iteration needs to be $m\geq \widetilde{\mathcal{O}}(d/\epsilon^2)$. Here the notation $\widetilde{\mathcal{O}}$ hides a $\log$ factor.

We should note, however, that in each iteration of LMC, one local gradient needs to be computed, and this is equivalent to a calculation of $d$ partial derivatives per iteration. This essentially means a cost of $\widetilde{\mathcal{O}}(d^2/\epsilon^2)$ is needed for one good sample. For a problem with high dimensionality $d\gg 1$, the cost is prohibitive. It would be desirable to combine this method with strategies that eliminate gradient computation for a gradient-free fast-converging sampling method.

\subsection{Ensemble mean gradient approximation}\label{sec:ega}
Ensemble sampling methods have been gaining ground in recent years. The idea is to evolve a large number of samples altogether so that samples could provide information to each other. In particular, if two samples are close to each other, the finite difference roughly provides approximate gradient information. There are various choices of using neighbors to find approximated gradients. We look for a probability ensemble in this article. Suppose we look for an approximate gradient of $f$ at $x^\star\in\mathbb{R}^d$ using its neighbors $x$ that are within $\eta$ distance, and assume the neighbor $x$ is drawn from an arbitrary probability density function $q(x)$, independent of $x^\star$, then call
\begin{equation}\label{eqn:def_d_tilde}
\tilde{d}_{\eta,q}(x^\star) =\alpha_d \frac{\left\langle \nabla f(x^\star),x-x^\star\right\rangle}{|x-x^\star|^2}\frac{\mathbf{1}_{|x-x^\star|\leq \eta}}{q(x)}(x-x^\star)\,,
\end{equation}
where $\alpha_d$ is the normalization constant:
\begin{equation}\label{eqn:def_alpha}
\alpha_d =\frac{d}{V}= \frac{d^2}{S_{d}\eta^d}\,,\quad\text{where}\quad V=\int_{|x-x^\star|\leq \eta} 1\rd x=\int^\eta_0 r^{d-1}S_ddr=\frac{\eta^dS_d}{d}\,,
\end{equation}
with $S_{d}$ being the volume of unit $d$-sphere, we can formulate an ensemble gradient approximation:
\begin{equation}\label{eqn:ensemble_approx}
\nabla f(x^\star)=\mathbb{E}_q\left(\tilde{d}_{\eta,q}(x^\star)\right)\,.
\end{equation}
The formula~\eqref{eqn:ensemble_approx} is valid merely because:
\begin{equation*}
\begin{aligned}
\nabla f(x^\star)&=\frac{d}{V}\int_{|x-x^\star|\leq\eta} \frac{\left(x-x^\star\right)\otimes \left(x-x^\star\right)}{|x-x^\star|^2}\rd x\cdot \nabla f(x^\star)\\
&=\alpha_d\int_{|x-x^\star|\leq\eta} \frac{\left(x-x^\star\right)\otimes \left(x-x^\star\right)}{|x-x^\star|^2}\rd x\cdot \nabla f(x^\star)\\
&=\alpha_d\int_{\mathbb{R}^d} \frac{\left\langle \nabla f(x^\star),x-x^\star\right\rangle}{|x-x^\star|^2}\frac{\mathbf{1}_{|x-x^\star|\leq \eta}}{q(x)}(x-x^\star)q(x)\rd x\\
&=\alpha_d\mathbb{E}_q\left(\frac{\left\langle \nabla f(x^\star),x-x^\star\right\rangle}{|x-x^\star|^2}\frac{\mathbf{1}_{|x-x^\star|\leq \eta}}{q(x)}(x-x^\star)\right)\,.
\end{aligned}
\end{equation*}

One key idea of the ensemble gradient approximation is to realize that the term in $\tilde{d}_{\eta,q}$ can be approximated when $\eta$ is small, namely:
\[
\langle\nabla f(x^*)\,,x-x^\star\rangle\approx f(x)-f(x^\star)\,.
\]
Replace the $\langle\nabla f\,,x-x^\star\rangle$ term in $\tilde{d}_{\eta,q}$ by the finite difference term, and define
\begin{equation}\label{eqn:def_d}
{d}_{\eta,q}(x^\star) = \alpha_d\frac{f(x)-f(x^\star)}{|x-x^\star|^2}\frac{\mathbf{1}_{|x-x^\star|\leq \eta}}{q(x)}(x-x^\star)\,,
\end{equation}
then the gradient $\nabla f(x^\ast)$ has a finite difference approximation, replacing~\eqref{eqn:ensemble_approx}:
\begin{equation}\label{error:discretization2}
\nabla f(x^\star)\approx \mathbb{E}_q(d_{\eta,q}(x^\star))=\mathbb{E}_q\left(\alpha_d\frac{f(x)-f(x^\star)}{|x-x^\star|^2}\frac{\mathbf{1}_{|x-x^\star|\leq \eta}}{q(x)}(x-x^\star)\right)\,.
\end{equation}
We can further justify the error in this approximation. Suppose $\nabla f$ is Lipschitz continuous, then
\begin{equation}\label{error:finitedifference}
\left|f(x)-f(x^\star)-\left\langle \nabla f(x^*),x-x^\star\right\rangle\right|\leq L|x-x^\star|^2 \leq L\eta^2\,,
\end{equation}
we have:
\begin{equation}\label{eqn:fd_error}
\begin{aligned}
&|\nabla f(x^\ast) - \mathbb{E}_q(d_{\eta,q}(x^\star))|\\
\leq&\EE_q\left(|d_{\eta,q}(x^\star)-\tilde{d}_{\eta,q}(x^\star)|\right)\\
=&\EE_q\left(\left|\alpha_d\frac{\left\langle \nabla f(x^\star),x-x^\star\right\rangle}{|x-x^\star|^2}\frac{\mathbf{1}_{|x-x^\star|\leq \eta}}{q(x)}(x-x^\star) - \alpha_d\frac{f(x)-f(x^\star)}{|x-x^\star|^2}\frac{\mathbf{1}_{|x-x^\star|\leq \eta}}{q(x)}(x-x^\star)\right|\right)\\
= & \EE_q\left(\left|\alpha_d\frac{\left|f(x)-f(x^\star)-\left\langle \nabla f(x^\star),x-x^\star\right\rangle\right|}{|x-x^\star|^2}\frac{\mathbf{1}_{|x-x^\star|\leq \eta}}{q(x)}(x-x^\star)\right|\right)\\
\leq &\EE_q\left(\left|\alpha_dL\frac{\mathbf{1}_{|x-x^\star|\leq \eta}}{q(x)}(x-x^\star)\right|\right)\leq L\eta d\,.
\end{aligned}
\end{equation}
This formula suggests the approximation is first order in $\eta$, and the smallness of $\eta$ needs to dominate the largeness in $d$.
\begin{remark}\label{rmk:ensemble_accurate}
We also stress that the derivation is valid only if the neighbors are distributed according to $q(x)$, a known distribution, and that this $q(x)$ needs to be independent of $x^\ast$.
\end{remark}

Suppose in reality, we have $N$ independent particles around $x^\star$, denoted as $\left\{x_j\right\}^N_{j=1}$, sampled from $q_j(x)$ respectively, then the ensemble gradient approximation formula is further reduced to:
\begin{equation}\label{eqn:firstwayapprox}
\nabla f(x^\star)\approx \alpha_d\frac{1}{N}\sum^N_{j=1}\frac{f(x_j)-f(x^\star)}{|x_j-x^\star|^2}\frac{\mathbf{1}_{|x_j-x^\star|\leq \eta}}{q_j(x_j)}(x_j-x^\star)\,.
\end{equation}
We note that $q_j(x)$ do not have to be the same.

\section{Algorithms and properties}\label{sec:alg}
We propose our new methods in this section. The strategy is to sample a large number of particles according to LMC~\eqref{eqn:LMCdiscretization}, and replace the gradients in LMC using the ensemble gradient approximation~\eqref{eqn:firstwayapprox}. Then immediately the samples are no longer i.i.d. but they share the same marginal distribution.

We discuss in Section~\ref{sec:EnLMC} the straightforward combination of the two. We term the method the Ensemble LMC (EnLMC). However, we will find the algorithm is rather unstable due to the gradient approximation in the unstable regions. This suggests us to enact the ensemble gradient approximation only in a constrained manner. The new algorithm, termed the Constrained Ensemble LMC (CEnLMC), will be discussed in Section~\ref{sec:CEnLMC}, in which we provide a number of constraints, and enact the ensemble gradient approximation only when these constraints are satisfied. The intuition of how these constraints are formulated will also be discussed. The theoretical results will also be summarized in Section~\ref{sec:convergenceofCEnLMC}.

\subsection{Ensemble LMC, a direct combination}\label{sec:EnLMC}
We now study the direct combination of LMC and the ensemble gradient approximation. Denote $\{x_i^m\}_{i=1}^N$ the $N$ samples at the $m$-th step iteration, then following the LMC formula, we would like to write Ensemble LMC (EnLMC) in the form of:
\begin{equation}\label{eqn:newLMC}
x^{m+1}_i=x^m_i-hF^m_i+\sqrt{2h}\xi^m_i\,,
\end{equation}
with the force $F^m_i=\frac{1}{N-1}\sum F^m_{ij}$ approximating $\nabla f(x^m_i)$. Here $F^m_{ij}$ stands for the contribution of $x^m_j$ towards calculating $\nabla f(x^m_i)$.

Denote $\mathcal{F}^{m-1}=\sigma \left(x^{n\leq m-1}_{j\leq N}\right)$ the filtration, and $p^m_j$ the marginal distribution of $x^m_j$ conditioned on $\mathcal{F}^{m-1}$, we can replace $x^\ast$ and $q(x)$ by $x^m_i$ and $p^m_j(x)$ respectively in~\eqref{eqn:def_d_tilde} to define:

\begin{equation}\label{eqn:d_tilde_ij}
G^m_{ij}=\alpha_d\frac{\left\langle \nabla f(x^{m}_i),x^m_j-x^{m}_i\right\rangle}{|x^m_j-x^{m}_i|^2}\frac{x^m_j-x^{m}_i}{\mathrm{p}^{m}_j}\mathbf{1}_{|x^m_j-x^{m}_i|\leq \eta}
\end{equation}
where $\mathrm{p}^m_j = p^{m}_j(x^{m}_j)$ and $\alpha_d$ is defined in \eqref{eqn:def_alpha}. Then, we still have ~\eqref{eqn:ensemble_approx} holds true, meaning, for all $j\neq i$, 
\begin{equation}\label{eqn:force_i_m}
\nabla f(x^m_i)=\EE_{p^m_j}(G^m_{ij})=\EE\left(G^m_{ij}\middle|\mathcal{F}^{m-1},x^m_i\right)\,.
\end{equation}

Recall the definition of $d_{\eta,q}$ in~\eqref{eqn:def_d}, we define
\begin{equation}\label{eqn:d_ij}
F^m_{ij}=\alpha_d\frac{\delta f^m_{ij}}{|\delta x^{m}_{ij}|^2}\frac{\delta x^{m}_{ij}}{\mathrm{p}^{m}_j}\mathbf{1}_{|\delta x^{m}_{ij}|\leq \eta}\,,\quad\text{with}\quad\begin{cases} \delta f^{m}_{ij} = f(x^m_j)-f(x^m_i)\,,\\ \delta x^m_{ij} = x^m_j-x^m_i\,,\end{cases}
\end{equation}
and thus, citing~\eqref{eqn:fd_error}, we have
\begin{equation}\label{eqn:error_GF}
\EE\left(\left|G^m_{i,j}-F^m_{i,j}\right|\middle|\mathcal{F}^{m-1},x^m_i\right)\leq L\eta d\,.
\end{equation}

Summing up contribution from all $j\neq i$, we approximate $\nabla f(x^m_i)$ by:
\begin{equation}\label{eqn:gradientapproximation}
\nabla f(x^m_i)\approx F^{m}_i=\frac{1}{N-1}\sum^N_{j\neq i}F^{m}_{ij}\,.
\end{equation}

We note that according to~\eqref{eqn:newLMC}, $\mathrm{p}^m_j= p_j^m(x^m_j)$ can be explicitly calculated. Indeed to update $x^m_j$ from $x^{m-1}_j$, we need $x^{m-1}_j$, $F^{m-1}_j$ and a random variable $\xi^{m-1}_j$. Realizing that when conditioned on $\mathcal{F}^{m-1}$, both $x^{m-1}_j$ and $F^{m-1}_j$ are determined, and the only randomness comes from the Gaussian variable $\xi^{m-1}_j$, meaning $x^m_j$ is merely a Gaussian variable as well when conditioned on  $\mathcal{F}^{m-1}$:
\[
x^m_j|\mathcal{F}^{m-1}\sim \mathcal{N}(x^{m-1}_j-hF^{m-1}_j\,,2hI_d)\,,
\]
or in other words:
\begin{equation}\label{eqn:p_j_x_j1}
p_j^m(x)=\frac{1}{(4\pi h)^{d/2}}\exp\left(-|x-\left(x^{m-1}_j-hF^{m-1}_j\right)|^2/(4h)\right)\,.
\end{equation}
Plugging in the definition of $x^m_j$, we can compute $\mathrm{p}^m_j$ explicitly:
\begin{equation}\label{eqn:p_j_x_j}
\mathrm{p}^m_j=\frac{1}{(4\pi h)^{d/2}}\exp\left(-|\xi^{m-1}_j|^2/2\right)\,.
\end{equation}

\begin{remark}
This is to resonate the discussion in Remark~\ref{rmk:ensemble_accurate}. In the derivation above we used the conditional distribution, conditioned on $\mathcal{F}^{m-1}$. If one uses~\eqref{eqn:firstwayapprox} in a brute-force manner, including all randomness, then we arrive at
\[
F^{m}_{ij}=\alpha_d\frac{\delta f^{m}_{ij}}{|\delta x^m_{ij}|^2}\frac{\mathbf{1}_{|\delta x^m_{ij}|\leq \eta}}{p^{m}(x^{m}_j)}\delta x^m_{ij}\,,
\]
where $p^m$ is the true distribution of $x^m_j$ without the conditioning. However, this definition of $F^m_{ij}$ cannot be used in the ensemble approximation: The $x^m_i$ and $x^m_j$ are not independent to each other and thus the ensemble $\EE_{p^m}(F^m_{ij})$ may not recover $\nabla f(x^m_i)$. More importantly, $p^m(x)$ is unknown in practice, making the calculation impossible.
\end{remark}

We plug~\eqref{eqn:p_j_x_j} into~\eqref{eqn:gradientapproximation} and run~\eqref{eqn:newLMC} for the update. The method is termed Ensemble Langevin Monte Carlo (EnLMC), as presented in Algorithm \ref{alg:EnOLMC}.
\begin{algorithm}[H]
\caption{\textbf{Ensemble Langevin Monte Carlo (EnLMC)}}\label{alg:EnOLMC}
\begin{algorithmic}
\State \textbf{Preparation:}
\State 1. Input: $h$ (time stepsize); $N$ (particle number); $\eta$ (parameter); $d$ (dimension); $M$ (stopping index); $\alpha_d$ \eqref{eqn:def_alpha}; $f(x)$.
\State 2. Initial:  $\left\{x^0_i\right\}^N_{i=1}$ i.i.d. sampled from an initial distribution induced by $q^0(x)$.

\State \textbf{Run: } \textbf{For} $m=0\,,1\,,\cdots\,M$

   \textbf{For} $i=1\,,2\,,\cdots\,,N$
   \begin{itemize}
   \item[--] Define
\begin{equation}\label{eqn:secondapproximation}
F^m_i=\frac{1}{N-1}\sum^N_{j\neq i}F_{ij}^m\,,\quad\text{with}\quad F_{ij}^m=\alpha_d\frac{\delta f^{m}_{ij}}{|\delta x^m_{ij}|^2}\frac{\mathbf{1}_{|\delta x^m_{ij}|<\eta}}{\mathrm{p}^{m}_j}\delta x^m_{ij}\,,
\end{equation}
where $\delta f^m_{ij}$ and $\delta x^m_{ij}$ are defined in \eqref{eqn:d_ij}.

\item[--] Draw $\xi^{m}_i$ from $\mathcal{N}(0,I_d)$;
\item[--] Update
\begin{equation}\label{eqn:update_alg1}
    \left\{
    \begin{aligned}
    &x^{m+1}_i=x^m_i-hF^m_i+\sqrt{2h}\xi^m_i\\
    &\mathrm{p}^{m+1}_i=\frac{1}{(4\pi h)^{d/2}}\exp\left(-|\xi^m_i|^2/2\right)
    \end{aligned}
    \right.\,.
\end{equation}
\end{itemize}
\textbf{end}

\textbf{end}

\State \textbf{Output:} $\{x^M_i\}^N_{i=1}$.
\end{algorithmic}
\end{algorithm}

The design of this algorithm follows straightforwardly from intuition: One replaces the gradient in LMC by the ensemble approximation using the neighbors' information. Since the difference between the true gradient and the ensemble approximation shrinks to zero as $\eta$, the neighboring range vanishes, one may incline to conclude that this method would converge also, as long as $\eta$ is small enough.

However, this is not true. This ensemble surrogate of the gradient induces strong instability to the algorithm. Indeed, $\xi^m_j$ is a Gaussian variable, and for every fixed $\epsilon$, there is non-trivial probability that makes $p_j^m(x^m_j)<\epsilon$, which blows up the force term~\eqref{eqn:secondapproximation}. We explicitly show this instability using the following example with $d=1$ and $f(x)=x^2/2$:

\begin{theorem}\label{thm:blowupvariance}
Assume $\left\{x^m_i\right\}^N_{i=1}$ are generated from Algorithm \ref{alg:EnOLMC}, then for $d=1$ and $f(x) = x^2/2$, we have: for any $m>0$, $1\leq i\leq N$
\begin{equation}\label{eqn:blowupvarianceinfinity}
\mathbb{E}|x^m_i|^2=\infty\,.
\end{equation}
\end{theorem}

This negative example suggests that directly replacing the gradient by the ensemble approximation leads to an unstable method.

We leave the proof to Section \ref{sec:proof_ENLMC}, but quickly discuss the intuition of the proof here. Indeed, to compute the variance of $x^{m+1}$ term: $\mathbb{E}|x^{m+1}_i|^2$, it is necessary to compute the variance of the force term $\mathbb{E}\left(\left|F^m_{i,j}\right|^2\right)$. The trajectory of $\{x_i\}_{i=1}^N$ is hard to trace, but one can nevertheless compute the conditional variance, conditioned on $\mathcal{F}^{m-1}$:
\begin{equation}\label{eqn:var_F}
\mathbb{E}\left(\left|F^m_{i,j}\right|^2\middle|\mathcal{F}^{m-1}\right)=\int \left|F^m_{i,j}\right|^2p^{m}_j(x^m_j)p^{m}_i(x^m_i)\rd x^{m}_j \rd x^m_{i}\,,
\end{equation}
where $p^m_i$ are the conditional probability distribution given $\mathcal{F}^{m-1}$.

Noting that according to the definition of $F^m_{ij}$ in~\eqref{eqn:secondapproximation}, for $f(x)=|x|^2/2$, we have:
\begin{equation}\label{eqn:F^m_ij_formula}
\begin{aligned}
F^m_{i,j}=\frac{1}{\eta}\frac{(x^{m}_j+x^{m}_i)(x^m_j-x^m_i)}{2|x^m_j-x^m_i|^2}\frac{\mathbf{1}_{|\delta x^m_{ij}|<\eta}}{p^{m}_j(x^m_j)}(x^m_j-x^m_i)=\frac{(x^{m}_j+x^{m}_i)}{2\eta}\frac{\mathbf{1}_{|\delta x^m_{ij}|<\eta}}{p^{m}_j(x^m_j)}\,.
\end{aligned}
\end{equation}
At the same time, denoting $w^m_i = x^{m-1}_i -h F^{m-1}_i$ the deterministic part of the update for $x^m_i$, we know that, for all $i$:
\begin{equation}\label{eqn:x^m_ij_formula}
x^m_i-w^m_i = \sqrt{2h}\xi^{m-1}_i\sim N(0,{2h})\quad\Rightarrow\quad p^{m}_i(x^{m}_i) = \exp\left(-\frac{|x^{m}_i-w^{m}_i|^2}{4h}\right)\,.
\end{equation}
Plugging~\eqref{eqn:F^m_ij_formula} and~\eqref{eqn:x^m_ij_formula} into~\eqref{eqn:var_F}, we have:
\begin{equation}\label{eqn:var_F_comp}
\begin{aligned}
&\mathbb{E}\left(\left|F^m_{i,j}\right|^2\middle|\mathcal{F}^{m-1}\right)\\
=&\int_{\mathbb{R}}\int_{B_\eta(x^m_i)}\frac{\left(x^{m}_j+x^m_i\right)^2}{4\eta^2}\exp\left(\frac{-{|x^{m}_i-w^{m}_i|^2}+|x^{m}_j-w^{m}_j|^2}{4h}\right)\rd x^{m}_j \rd x^m_{i}\,.
\end{aligned}
\end{equation}
Since the $p^{m}_j$ term is in the denominator in~\eqref{eqn:F^m_ij_formula}, and when one takes the variance, this term gets squared. In the end this exponential term from $x^m_j$ appears in a positive manner in~\eqref{eqn:var_F_comp}. This already suggests the blowing up of this variance term. A more careful derivation shows:
\begin{equation}\label{eqn:var_F_comp2}
\begin{aligned}
&\mathbb{E}\left(\left|F^m_{i,j}\right|^2\middle|\mathcal{F}^{m-1}\right)\\
=&\int_{\mathbb{R}}e^{-\frac{|x^{m}_i-w^{m}_i|^2}{4h}}\int_{B_\eta(0)}\frac{\left(z+2x^m_i\right)^2}{4\eta^2}e^{\frac{|z+x^{m}_i-w^{m}_j|^2}{4h}}\rd z \rd x^m_{i}\\
=&\int_{B_\eta(0)}e^{\frac{-|w^{m}_i|^2+|z-w^m_j|^2}{4h}}\int_{\mathbb{R}}\frac{\left(z+2x^m_i\right)^2}{4\eta^2}e^{\frac{x^m_i(z+w^m_{i}-w^m_j)}{2h}}\rd x^m_{i}\rd z\\
=&\infty\,.
\end{aligned}
\end{equation}
In the second equality we used the change of variables $z=x^m_j-x^m_i$. The infinity comes from the inner integral, where we are essentially looking at the second moment of an exponential function.

This infinite variance of $F^m_{i,j}$, calculated in~\eqref{eqn:var_F_comp2}, suggests the variance of $x^{m+1}_i$, to be showed in~\eqref{eqn:blowupvarianceinfinity}, is also infinite. Proving Theorem~\ref{thm:blowupvariance} then amounts to carrying out the detailed derivation on how $\mathbb{E}|x^{m+1}_i|^2$ depends on $\mathbb{E}\left|F^m_{i,j}\right|^2$, and we leave this to Section \ref{sec:proof_ENLMC}.

\subsection{Constrained Ensemble LMC, a modification}\label{sec:CEnLMC}
We now take a more careful look at the instability in the ensemble gradient approximation to LMC. Intuitively there are two sources of instability:
\begin{itemize}
\item When $x^m_i$ is at the ``outskirt" of $p(x)$, $f(x^m_i)$ is high, and $p(x^m_i) \propto \exp\{-f(x^m_i)\}$ is extremely small. This could bring high relative error, and we should avoid making any approximations in this region.
\item In the formula~\eqref{eqn:gradientapproximation}, $p^m_j(x^m_j)$ is in the denominator. Considering the way the term is defined in~\eqref{eqn:p_j_x_j}, it takes an $\mathcal{O}(1)$ value with high probability when $\xi^m_j$ is moderately small. However, there is a small chance for $|\xi^m_j|$ to take large values, which will make $p^m_j(x^m_j)$ extremely small, bringing infinite variance, as shown in~\eqref{eqn:var_F_comp2}.
\end{itemize}

To avoid these two scenarios, we essentially need to identify:
\begin{itemize}
\item $x^m_i$ who are at the ``outskirt" of $p$;
\item $x^m_j$ that is within $\eta$ distance from $x^m_i$ but has large $|\xi^{m-1}_j|$.
\end{itemize}
When these happen, the ensemble approximation is disabled and we come back to use the true gradient $\nabla f(x^m_i)$.

To identify the first scenario is relatively straightforward: We simply set a threshold, call it $M_f$, and will only employ ensemble gradient approximation when $f(x^m_i)$ is smaller than $M_f$:
\[
\textcolor{black}{f(x^m_i)< M_f}\,.
\]

To identify the second scenario is slightly more involved. We now consider
\[\begin{aligned}
\sqrt{2h}|\xi^{m-1}_j|=|x^{m}_j-w^{m}_j|\leq& |x^{m}_j-x^{m}_i|+|x^{m}_i-w^{m}_i|+|w^{m}_i-w^{m}_j|\\
=&|\delta x^{m}_{ij}|+\sqrt{2h}|\xi^{m-1}_i|+|\delta w^{m}_{ij}|
\end{aligned}
\]
where we denote the deterministic component of the updating formula:
\begin{equation}\label{eqn:w}
w^{m}_i=x^{m-1}_i-hF^{m-1}_i\,,\quad \delta w^m_{ij}=w^m_j-w^m_i\,.
\end{equation}
A sufficient condition to have a moderate $|\xi^{m-1}_j|$ is to have all three terms on the right hand side moderate. For a fixed $x^{m}_i$, since we only consider $x^{m}_j$ who are already within $\eta$ distance, the first term is already bounded by $\eta$ and is small. We therefore need to ensure the remaining two terms are bounded as well. To do so, we propose to enact the ensemble gradient approximation only if $|\xi^{m-1}_i|$ is at most moderately large, and for those $x^m_i$, we include the $x^m_j$ contribution in the calculation of $F^m_i$ only if $|\delta w^{m}_{ij}|$ is at most moderately large. This is to say, for a fixed preset constant pairs $(R_1\,,R_2)$:
\begin{itemize}
\item When $\sqrt{2h}|\xi^{m-1}_i|>R_1$:
\begin{equation}\label{eqn:constrain1}
F^m_i = \nabla f(x^m_i)\,,
\end{equation}
\item When $\sqrt{2h}|\xi^{m-1}_i|\leq R_1$:
\begin{equation}\label{eqn:Fmi_cenlmc}
F^m_i=\frac{1}{N^m_i}\sum_{j\neq i}^NF^m_{ij}\,,\quad\text{with}\quad F^m_{ij} =\alpha_d\frac{\delta f^m_{ij}}{|\delta x^m_{ij}|^2 }\frac{\delta x^m_{ij}}{\mathrm{p}^{m}_j}\mathbf{1}_{|\delta x^m_{ij}|\leq\eta\,,|\delta w^m_{ij}|\leq R_2}\,,
\end{equation}
where $\mathrm{p}^{m}_j$ is defined in \eqref{eqn:p_j_x_j} and
\begin{equation}\label{eqn:N_neighbors}
N^m_i=\sum^N_{j\neq i}\mathbf{1}_{|\delta w^m_{ij}|\leq R_2}\,,
\end{equation}
is the number of neighbors within $\eta$ distance whose corresponding $|\delta w^m_{ij}|$ is controlled. 
\end{itemize}
Note that compared with~\eqref{eqn:d_ij}, we add another indicator function in~\eqref{eqn:Fmi_cenlmc} to ensure $\delta w^m_{ij}$ is controlled by $R_2$. Furthermore, numerically to have statistical stability, we also preset a value for $N^\ast$ and require $N^m_i\geq N^\ast$. If $N^m_i<N^\ast$, we do not enact the ensemble approximation and use the true gradient $\nabla f(x^m_i)$.

Summarizing the discussion above, we have:
\begin{equation}\label{eqn:Fmibetter}
F^m_i=\left\{
\begin{aligned}
&\nabla f(x^m_i),\quad  &\sqrt{2h}|\xi^{m-1}_i|>R_1\;\text{or}\; \textcolor{black}{f(x^m_i)> M_f}\;\text{or}\; N^\ast> N^m_i\\
&\frac{1}{N^m_i}\sum^N_{j\neq i}F_{ij}^m\,,\quad &\text{otherwise}\,.
\end{aligned}
\right.
\end{equation}

Replacing the gradient term in LMC using~\eqref{eqn:Fmibetter}, we arrive at a new algorithm. We term it Constrained Ensemble Langevin Monte Carlo (CEnLMC), as summarized in Algorithm \ref{alg:CEnOLMC}.

\begin{algorithm}[htbp]
\caption{\textbf{Constrained Ensemble Langevin Monte Carlo (CEnLMC)}}\label{alg:CEnOLMC}
\begin{algorithmic}
\State \textbf{Preparation:}
\State 1. Input: $h$ (time stepsize); $N$ (particle number); $\eta,R_1,R_2,N^\ast,M_f$ (parameters); $d$ (dimension); $M$ (stopping index); $\alpha_d$ \eqref{eqn:def_alpha}; $\nabla f(x)$; $f(x)$; $f^\ast$ (minimal value).

\State 2. Initial:  $\left\{x^0_i\right\}^N_{i=1}$ i.i.d. sampled from an initial distribution induced by $q^0(x)$.

Set $w^{-1}_i=\infty$ for $1\leq i\leq N$.

\State \textbf{Run: }\textbf{For} $m=0\,,1\,,\cdots\,,M$

   \textbf{For} $i=1\,,2\,,\cdots\,,N$
\begin{itemize}
\item[--] Define
\[
N^m_i=\sum^N_{j\neq i}\mathbf{1}_{|\delta w^m_{ij}|<R_2}\,.
\]
\item[--]
  \textbf{If} $\sqrt{2h}|\xi^{m-1}_i|> R_1$ or \textcolor{black}{$f(x^m_i)> M_f$} or $N^\ast>N^m_i$, define
\[
F^m_i=\nabla f(x^m_i)\,.
\]
\textbf{else} define
\begin{equation}\label{eqn:thirdapproximation}
F^m_i=\frac{1}{N^m_i}\sum^N_{j\neq i}F^m_{ij}\,,\quad\text{with}\quad F^m_{ij}=\alpha_d\frac{\delta f^m_{ij}}{|\delta x^m_{ij}|^2}\frac{\delta x^m_{ij}}{\mathrm{p}^{m}_j}\mathbf{1}_{|\delta x^m_{ij}|\leq\eta\,,|\delta w^m_{ij}|\leq R_2}\,.
\end{equation}
where $\delta f^m_{ij}$, $\delta x^m_{ij}$ are defined in \eqref{eqn:d_ij}, and $\delta w^m_{i,j}$ is defined in \eqref{eqn:w}.

\textbf{end}

\item[--] Draw $\xi^{m}_i$ from $\mathcal{N}(0,I_d)$.
\item[--] Update 
\begin{equation}\label{eqn:algcelmc}
    \left\{
    \begin{aligned}
    &x^{m+1}_i=x^m_i-hF^m_i+\sqrt{2h}\xi^m_i\,,\\
    &\mathrm{p}^{m+1}_i=\frac{1}{(4\pi h)^{d/2}}\exp\left(-|\xi^m_i|^2/2\right)\,,\\
    &w^{m+1}_i=x^m_i-hF^m_i
    \end{aligned}
    \right.
\end{equation}
\end{itemize}
\textbf{end}

\textbf{end}
\State \textbf{Output:} $\{x^M_i\}^N_{i=1}$.
\end{algorithmic}
\end{algorithm}
\subsection{Properties of CEnLMC}\label{sec:convergenceofCEnLMC}
There are two types of properties of CEnLMC that we would like to discuss: 1. the convergence: We would like to show that the distribution of $x^m_i$, as $m\to\infty$ converges to the target distribution; 2. the numerical cost: We would like to show that the probability of computing the gradients is low with a proper tuning of $R_1$, $R_2$ and $M_f$, and thus most gradients are replaced by its cheaper ensemble version. This makes CEnLMC cheaper than the classical LMC. 

These two properties are discussed in the following subsections respectively.

\subsubsection{Convergence of CEnLMC}
To show the method converges is to show that the distribution of $x^m_i$, as $m\to\infty$, converges to the target distribution $p$ up to a small discretization error.

Our strategy is to show that particles computed from CEnLMC are close to the particles computed from the classical LMC if they start with the same initial data. Since it is well-known that the distribution of LMC samples converges to the target distribution, the samples found by CEnLMC then recover the target distribution as $m\to\infty$ as well.

We first introduce the particle system that solves the classical LMC~\eqref{eqn:LMCdiscretization}. Define $z^0_i=x^0_i$ for $1\leq i\leq N$ and update
\begin{equation}\label{eqn:lmc}
z^{m+1}_i=z^m_i-\nabla f(z^m_i)h+\sqrt{2h}\xi^m_i\,,
\end{equation}
where $\xi^m_i$ is the same as \eqref{eqn:algcelmc}. This is the classical LMC algorithm, and all samples $z_i$ are decoupled from each other. Our first goal is to show that $x^m_i$ and $z^m_i$ are approximately the same, as seen in the following theorem.
\begin{theorem}\label{thm:convergenceofCEnLMC_z} 
Assume $\left\{x^m_i\right\}^N_{i=1}$ are generated from Algorithm \ref{alg:CEnOLMC}, and $\left\{z^m_i\right\}^N_{i=1}$ are generated from~\eqref{eqn:lmc}, with the parameters chosen to satisfy
\[
h\leq \min\left\{\frac{1}{L},\ \frac{1}{d}\right\}\,,\; \max\{\eta,1\}\leq R_2\,,\; \textcolor{black}{M_f>f^*}\,,
\]
\textcolor{black}{where $f^*$ is the optimal (minimum) of $f(x)$.} Assume $f$ is $L$-smooth, then, for $m\geq 0$, $1\leq i\leq N$:
\begin{equation}\label{eqn:thm:disxznonconvex}
\mathbb{E}|x^m_i-z^m_i|\leq \mathcal{O}\left(\exp(Lmh)\left(\sqrt{\frac{R^{d}_1\textcolor{black}{(M_f-f^*)}d^2}{L\eta^{d}N^\ast}}\exp\left(\frac{R_2(R_2+R_1)}{2h}\right)+\eta d\right)\right)\,.
\end{equation}
If we further assume $f$ is $\mu$-convex, then, denoting $\kappa = L/\mu$, for any $m\geq0$, $1\leq i\leq N$:
\begin{equation}\label{eqn:thm:disxz}
\mathbb{E}|x^m_i-z^m_i|\leq \mathcal{O}\left(\sqrt{\frac{R^{d}_1\kappa \textcolor{black}{(M_f-f^*)}d^2}{\mu \eta^{d}N^\ast}}\exp\left(\frac{R_2(R_2+R_1)}{2h}\right)+\kappa \eta d\right)\,.
\end{equation}
\end{theorem}

We leave the proof to Section~\ref{pfthCENLMC}.

\textcolor{black}{We stress the importance of this theorem. The theorem estimates the distance between the proposed samples and the classical LMC samples. With the properly tuned parameters, we can make the bound in (37)-(38) small, forcing the two sets of samples close to each other. LMC is a classical algorithm that we have rich understanding about. In particular, we have results from~\cite{DALALYAN20195278,dalalyan2018sampling,durmus2017} that give non-asymptotic error estimate: The error, in Wasserstein distance, converges to zero, exponentially fast, up to the discretization error that depends on $d$, the dimension of the problem, and $h$, the stepsize. This means, the newly proposed algorithm CEnLMC also converges exponentially fast, up to the discretization error and this newly induced approximation error.}

\textcolor{black}{We now take a closer look at this approximation error. Use the convex case as an example, we examine the two terms in \eqref{eqn:thm:disxz}. The second bound mainly comes from the finite difference approximation, induced in~\eqref{eqn:error_GF}, and the first term traces back to ensemble error ($\EE|\nabla f(x^m_i)-G^m_i|^2$). After adding constraints \eqref{eqn:constrain1}-\eqref{eqn:Fmibetter}, this error contributes to $1/\sqrt{N^\ast}$ term. This is optimal in terms of $N^\ast$ according to the central limit theorem.}

To make the distance small, we first need to let $\eta$ be small so that the error from the finite differencing is small. Upon choosing small $\eta$, with $R_{1,2}$ fixed, we need to select a moderate $\textcolor{black}{(M_f-f^*)}/N^\ast$ to make the first term small. Since $M_f$ is the bound we set to turn on or off the ensemble gradient approximation, we expect it to be relatively large. $N^\ast$ is the minimum number of neighbors needed to enact the ensemble approximation to ensure statistical accuracy and is thus also expected to be large. To accommodate both, we set \textcolor{black}{$M_f = (N^\ast)^\rho+f^*$} with $\rho<1$.

We summarize this choice of parameters in the following corollary:

\begin{corollary}\label{cor:disxznonconvex}
Under the same assumption as in Theorem~\ref{thm:convergenceofCEnLMC_z} and let $f$ be $\mu$-convex, for any small number $\epsilon>0$ and $0<\rho<1$, by setting 
\begin{equation}\label{eqn:onecondition}
\textcolor{black}{M_f = (N^\ast)^{\rho}+f^*},\quad \eta<\frac{\epsilon}{\kappa d},\quad N^\ast=\frac{R^{d/(1-\rho)}_1\kappa^{1/(1-\rho)}d^{2/(1-\rho)}}{\mu^{1/(1-\rho)}\eta^{d/(1-\rho)}\epsilon^{2/(1-\rho)}}\exp\left(\frac{R_2(R_2+R_1)}{2(1-\rho)h}\right)\,,
\end{equation}
we have: for any $m\geq0$, $1\leq i\leq N$:
\begin{equation}\label{eqn:thm:disxzcor}
\mathbb{E}|x^m_i-z^m_i|\leq \mathcal{O}\left(\epsilon\right)\,.
\end{equation}
\end{corollary}
This is obtained by simply setting both terms in~\eqref{eqn:thm:disxz} smaller than $\epsilon$. We omit the proof.

Now we are ready to combine this result with the well-known convergence result of LMC to show the convergence of CEnLMC. The convergence is discussed in both Wasserstein distance sense, and weak sense.
\begin{theorem}\label{thm:convergenceofCEnLMC}
Under the same assumption as in Theorem~\ref{thm:convergenceofCEnLMC_z} and let $f$ be $\mu$-convex, we denote $\kappa=L/\mu$ the condition number, $q^m_i$ the probability density of $x^m_i$. Assume $\int |x|q^0\rd{x}<\infty$, we have:
\begin{enumerate}[leftmargin=0cm,itemindent=1cm,align=left]
\item{$W_1$ convergence:}  For any $m\geq 0$, $1\leq i\leq N$,
\begin{equation}\label{eqn:convergenceofCEnLMC}
\begin{aligned}
W_1(q^m_i,p)\leq& \exp\left(-\frac{\mu h m}{2}\right)W_{1}(q^0,p)\\
&+\mathcal{O}\left(\kappa (\sqrt{hd}+\eta d)+\sqrt{\frac{R^{d}_1\kappa d^2\textcolor{black}{(M_f-f^*)}}{\mu\eta^{d}N^\ast}}\exp\left(\frac{R_2(R_2+R_1)}{2h}\right)\right)\,.
\end{aligned}
\end{equation}
\item{Weak convergence:} For any Lipschitz function $g:\mathbb{R}^d\rightarrow \mathbb{R}$ with $\EE_p(g^2)<\infty$ and $m\geq 0$, we have
\begin{equation}\label{eqn:convergenceofCEnLMC2}
\begin{aligned}
&\EE\left|\frac{1}{N}\sum^N_{i=1}g(x^m_i)-\EE_{p}(g)\right|\\
\leq &\mathcal{O}\left(\exp\left(-\frac{\mu h m}{2}\right)W_{1}(q^0,p)\right)\\
+&\mathcal{O}\left(\frac{1}{\sqrt{N}}+ \kappa (\sqrt{hd}+\eta d)+\sqrt{\frac{R^{d}_1\kappa d^2\textcolor{black}{(M_f-f^*)}}{\mu\eta^{d}N^\ast}}\exp\left(\frac{R_2(R_2+R_1)}{2h}\right)\right)\,.
\end{aligned}
\end{equation}
\end{enumerate}
\end{theorem}

We leave the proof to Section \ref{pfthCENLMC}. We note that in both~\eqref{eqn:convergenceofCEnLMC} and~\eqref{eqn:convergenceofCEnLMC2}, there is one exponentially decaying term, and the rest can be seen as the remainder term. Therefore we can call the convergence rate exponential, up to a controllable discretization and ensemble error. The exponentially decaying term comes from the fact that the distribution of $z^m_i$ decays to the target distribution exponentially fast, and the remainder term mostly comes from the distance between $\{x^m_i\}$ and $\{z^m_i\}$ systems.

\textcolor{black}{\begin{remark}\label{rem:mf} This theorem gives a clear guidance on the choice of some parameters. To have fast convergence and small error term, the parameters need to be tuned to have second term in \eqref{eqn:convergenceofCEnLMC} as small as possible. Assume we have enough particles ($N\to\infty$), we set this term to be smaller than $\epsilon$, then:
\[
\eta\leq \mathcal{O}\left(\frac{\epsilon}{\kappa d}\right),\ h\leq\mathcal{O}\left(\frac{\epsilon^2}{36\kappa^2d}\right),\ N^*>\mathcal{O}\left(\frac{36R^{d}_1\kappa d^2(M_f-f^*)}{\mu\eta^{d}\epsilon^2}\exp\left(\frac{R_2(R_2+R_1)}{h}\right)\right)\,.
\]
We then set the first term to be smaller than $\epsilon$ as well, then the lower bound for the needed number of iteration is:
\[
m>\mathcal{O}\left(\frac{\kappa^2 d}{\epsilon^2}\log\left(\frac{W_1(q^0,p)}{\epsilon}\right)\right)\,,
\]
meaning after these many iterations, $W_1(q^m_i,p)\leq 2\epsilon$, where $q^m_i$ is the distribution of $x^m_i$.
\\
Note that this gives the control of $\eta$, $h$ and $N^\ast$ but still leaves the freedom to adjust $R_1$, $R_2$ and $M_f$. These parameters should be determined by the percentage of gradient that we are willing to calculate. The discussion is found in Remark~\ref{rem:R}.
\end{remark}}

\subsubsection{Numerical saving of CEnLMC}\label{sec:saving}
We now discuss the numerical saving of CEnLMC compared with the classical LMC.

The main reason to utilize the ensemble gradient approximation is to avoid the gradient computation. In the algorithm, the ensemble approximation is enacted only if:
$$
\sqrt{2h}|\xi^{m-1}_i|\leq R_1\,,\quad \textcolor{black}{f(x^m_i)\leq M_f}\,,\quad N^m_i\geq N^\ast\,,
$$
where the size of $N^m_i$ depends on the number of samples who satisfy $|\delta w^m_{ij}|\leq R_2$. Therefore the probability of not using the ensemble approximation (but using $\nabla f$) can be bounded by:

\begin{equation}\label{eqn:prob_bound}
\begin{aligned}
\mathbb{P}\left(\left\{F^m_i = \nabla f(x^m_i)\right\}\right)\leq &\mathbb{P}\left(\left\{\sqrt{2h}|\xi^m_i|>R_1\right\}\right)\\
&+\mathbb{P}\left(\left\{|f(x^m_i)-f^\ast|>\textcolor{black}{(M_f-f^*)}\right\}\right)\\
&+\mathbb{P}\left(\left\{N^m_i<N^\ast\right\}\right)\,.
\end{aligned}
\end{equation}

One thus needs to choose the parameters wisely to make such a probability as small as possible so that most gradients in LMC get replaced by the ensemble approximation. More specifically, we have the following theorem:
\begin{theorem}\label{thm:ratioofcalculation}
Under the same assumption as in Theorem~\ref{thm:convergenceofCEnLMC_z} and let $f$ be $\mu$-convex. If $\mathrm{KL}(q_0|p)<\infty$, then for fixed $M\geq0$, we have:
\begin{align}
&\lim_{\eta\rightarrow0}\lim_{N\rightarrow\infty}\sup_{0\leq m\leq M,1\leq i\leq N}\mathbb{P}\left(\left\{\sqrt{2h}|\xi^{m-1}_i|>R_1\right\}\right)\leq C_d(R_1)\label{boundofP1}\,,\\
&\lim_{\eta\rightarrow0}\lim_{N\rightarrow\infty}\sup_{0\leq m\leq M,1\leq i\leq N}\mathbb{P}\left(\left\{|f(x^m_i)-f^\ast|>\textcolor{black}{(M_f-f^*)}\right\}\right)\leq\frac{2\kappa d}{\textcolor{black}{(M_f-f^*)}}\,,\label{boundofP2}\\
 &\lim_{\eta\rightarrow0}\lim_{N\rightarrow\infty}\sup_{0\leq m\leq M,1\leq i\leq N}\mathbb{P}\left(\left\{N^m_i<N^\ast\right\}\right)=0\,.\label{boundofP3}
\end{align}
where
\[
C_d(R_1)=\frac{S_d}{(2\pi)^{d/2}}\int^\infty_{\frac{R_1\sqrt{d}}{\sqrt{2}}}r^{d-1}\exp\left(-\frac{r^2}{2}\right)\rd r
\]
diminishes to $0$ for large $R_1$ and $S_d$ is the volume of unit $d$-sphere.
\end{theorem}

We leave the proof of the theorem to Section \ref{pfthm:ratioofcalculation}.  This theorem gives the bound to ~\eqref{eqn:prob_bound}. According to the formula of~\eqref{boundofP1}-\eqref{boundofP3}, a direct corollary is the following:
\begin{corollary}\label{cor:ratio}
Under the same assumption as in Theorem~\ref{thm:ratioofcalculation}, for any $\epsilon>0$, there exists constants $R^*,F^*$ only depend on $\epsilon,d$ such that if
\[
R_1>R^*,\quad M_f>F^*\,,
\]
we have
\[
\lim_{\eta\rightarrow0}\lim_{N\rightarrow\infty}\sup_{0\leq m\leq M,1\leq i\leq N}\mathbb{P}\left(\left\{F^m_i = \nabla f(x^m_i)\right\}\right)\leq \epsilon\,.
\]
\end{corollary}
According to the Corollary \ref{cor:ratio}, when we have enough particles, we can always tune the parameters so that most gradients in LMC get replaced by the ensemble approximation.

\textcolor{black}{\begin{remark}\label{rem:R}
This theorem gives the guideline for the parameter choice of $R_1$, $R_2$ and $M_f$. Suppose the percentage of the gradient we would like to compute is $\alpha$, and we equally distribute it to the three terms in (43). Then in the limit of $\eta\to0$ and $N\to\infty$, $R_1$ should be chosen, according to (44), so that
\[
C_d(R_1) \leq \frac{\alpha}{3}.
\]
Similarly, according to (45), $M_f$ should be chosen so that
\[
M_f\geq \frac{6\kappa d}{\alpha}+f^*\,.
\]
\\
Lastly, we need to give a bound for $R_2$. This can be implicitly computed from (46). While it is true that in the $N\to\infty$ limit, the probability is necessarily $<\frac{\alpha}{3}$, for every fixed $N$, the size of $R_2$ will affect the probability. Such dependence is very delicate, and we only give a rough bound. Suppose we are in the ideal case with $h\to0$ so that $x_i^m=w^m_i$, and suppose we have iterated many times and the particles are approximately close to i.i.d. sampled from the target distribution. then
\[
\begin{aligned}
\mathbb{P}\left(\left\{N^m_i<N^\ast\right\}\right)&= \mathbb{P}\left(\#\left\{w^m_j\middle||w^m_j-w^m_i|<R_2,\ j=1,2,\dots,N\right\}<N^*+1\middle|w^m_i\right)\\
&\approx \mathbb{P}\left(\#\left\{x^m_j\middle||x^m_j-x^m_i|<R_2,\ j=1,2,\dots,N\right\}<N^*+1\middle|x^m_i\right)\\
&= \sum_{k=0}^{N^*-1} {N-1 \choose k}p^k(R_2)(1-p(R_2))^{N-1-k}\ll O(1)
\end{aligned} 
\]
where $p(R_2)=\mathbb{P}_{y,z\sim p}(|y-z|<R_2)$. The first equation comes from the definition, and the second is driven by the fact that $x_i^m$ and $w_i^m$ are close by. Assuming $N^*<\frac{N+1}{2}, p(R_2)<\frac{1}{4}$, then
\[
\begin{aligned}
\mathbb{P}\left(\left\{N^m_i<N^\ast\right\}\right)&\approx \sum_{k=0}^{N^*-1} {N-1 \choose k}p^k(R_2)(1-p(R_2))^{N-1-k}\\
&\leq (1-p(R_2))^{N-1}{N-1 \choose N^*-1}\sum_{k=0}^{N^*-1}\left(\frac{p(R_2)}{1-p(R_2)}\right)^k\\
&\leq {N-1 \choose N^*-1}\frac{(1-p(R_2))^{N}}{1-2p(R_2)}\\
&\leq CN^{N^*}(1-p(R_2))^{N}
\end{aligned} 
\]
where $C$ is a uniform constant and we use Stirling's approximation in the last inequality. To have this term controlled by $\frac{\alpha}{3}$, we need to choose $p(R_2)$ so that:
\[
1-\left(\frac{\alpha}{3CN^{N^*}}\right)^{1/N}\leq p(R_2)\leq \frac{1}{4}\,,
\]
which permits:
\[
\mathbb{P}\left(\left\{N^m_i<N^\ast\right\}\right)\approx \sum_{k=0}^{N^*-1} {N-1 \choose k}p^k(R_2)(1-p(R_2))^{N-1-k}\leq \frac{\alpha}{3}\,.
\]
\end{remark}}

\section{Numerical experiment}\label{sec:numerics}
\textcolor{black}{We show two numerical examples to demonstrate the two main themes of the paper: the samples capture the target distribution, and the number of gradient calculations is significantly reduced. In particular, for both examples, we define the percentage of the gradient calculations:
\[
\mathcal{R}_m=\frac{\#\{F^j_i=\nabla f(x^j_i)|1\leq i\leq N, 1\leq j\leq m\}}{mN}\,,
\]
and we will show the evolution of this percentage in iterations. To demonstrate the accuracy, we also show the samples generated from LMC~\cite{roberts1996} and MALA (Metropolis-adjusted Langevin algorithm)~\cite{LDMHA2002,doi:10.1137/19M1284014}.}

\textcolor{black}{
\textbf{Example 1.} In this example, we set $d=2$, and the target distribution $p(x)\propto \exp(-|x_1|^2/2-|x_2|^2/8)$. Suppose the initial distribution is:
\[
q^0(x)\propto \exp\left(-\frac{(x_1-1)^2}{2}-\frac{(x_2-1)^2}{2}\right)+\exp\left(-\frac{(x_1+1)^2}{2}-\frac{(x_2+1)^2}{2}\right)\,.
\]
In the experiment, we choose $R_1=\frac{3\sqrt{5}}{10}$, $h=\eta=0.1$, $R_2=1.5$, $M_f=20$, and $N^\ast=10^3$. In Figure~\ref{Figure1}-\ref{Figure2}, we plot the samples generated by CEnLMC, LMC, and MALA at different iterations, using $N=10^4$. Since the example is logconcave in nature, the samples converge fairly quickly. Furthermore, we plot the ratio $\mathcal{R}_m$ at different iteration, using $N=2\times10^3, 6\times10^3, 10^4$, in Figure~\ref{Figure3}. While in the case of $N=2\times 10^3$, most particles need to have its gradient computed in every iteration, the ratio drops significantly for the larger $N$, and as iteration $m$ increases, the percentage of gradient calculation continues to decrease. This saving verifies the prediction from Section \ref{sec:saving}.
\begin{figure}[!ht]
     \centering
     \includegraphics[width=1\textwidth,height=0.4\textheight]{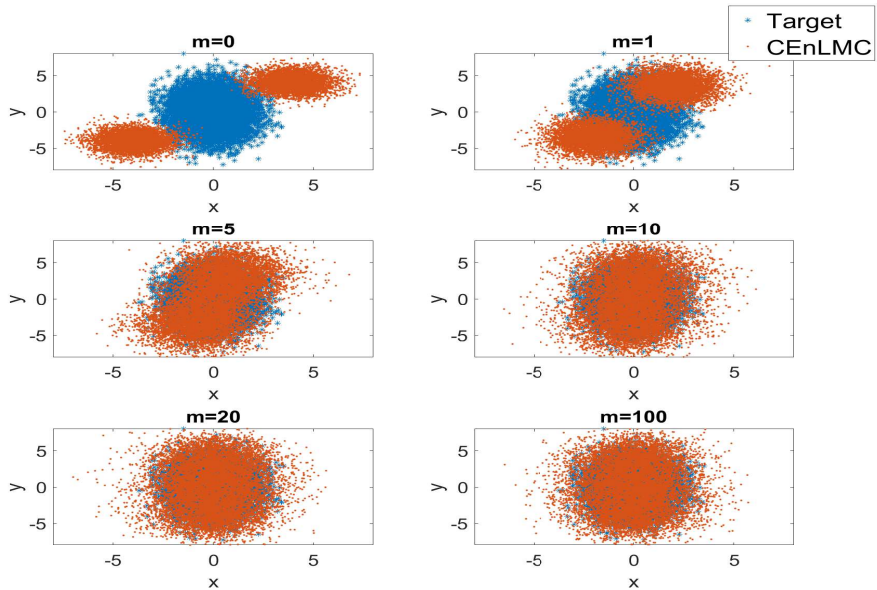}
     \caption{Example $1$: Evolution of samples using CEnLMC. $N=10^4$.}
     \label{Figure1}
\end{figure}
\begin{figure}[!ht]
     \centering
     \includegraphics[width=1\textwidth,height=0.4\textheight]{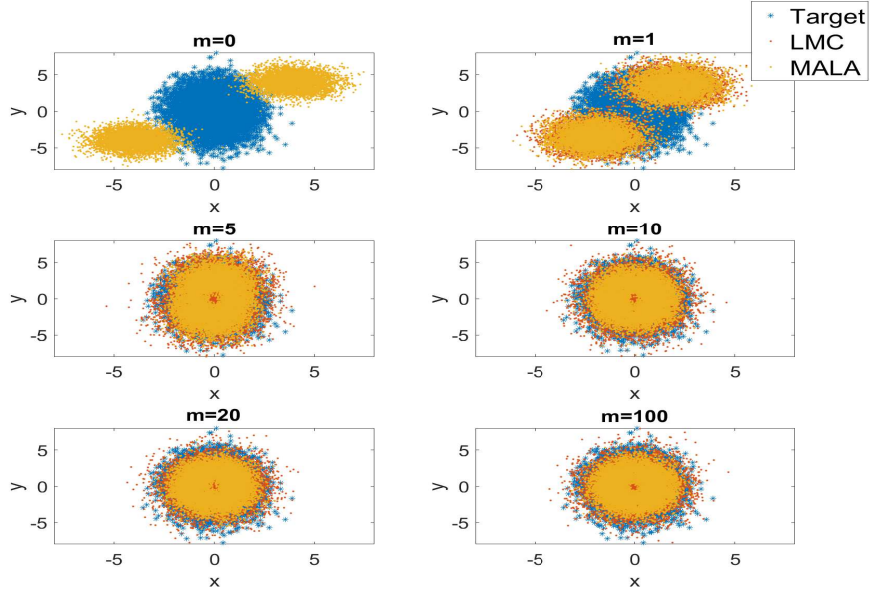}
     \caption{Example $1$: Evolution of samples using LMC and MALA. $N=10^4$.}
     \label{Figure2}
\end{figure}
\begin{figure}[!ht]
     \centering
     \includegraphics[width=1\textwidth,height=0.4\textheight]{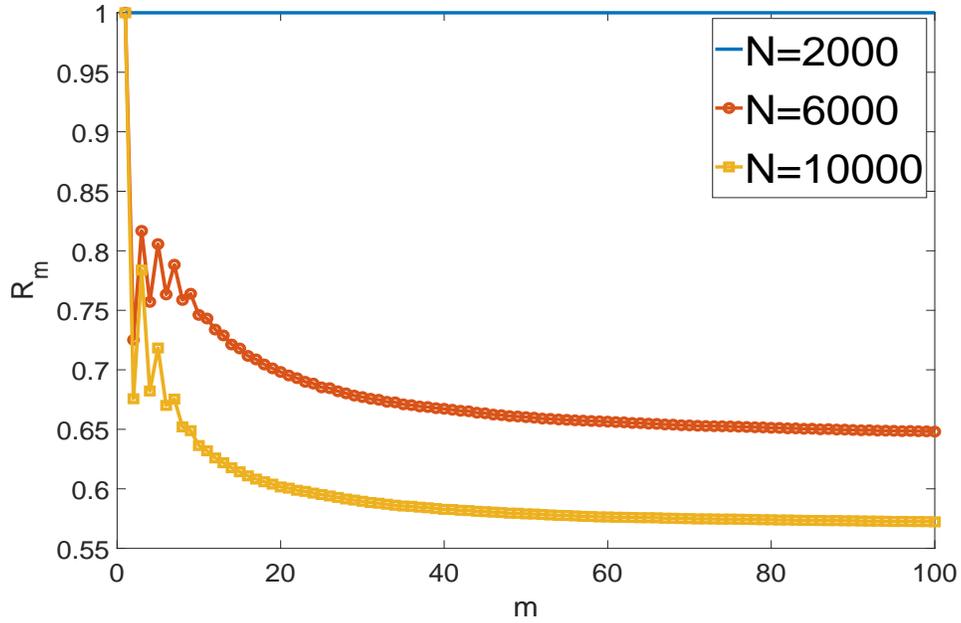}
     \caption{Example $1$: Evolution of $\mathcal{R}_m$ when $N=2\times10^3, 6\times10^3$ or $10^4$.}
     \label{Figure3}
\end{figure}}

\textcolor{black}{
\textbf{Example 2.} In this example, we test the algorithms on a target distribution that is not logconcave. Set the target to be
\[
p(x)\propto \exp\left(-\frac{(x_1-4)^2}{2}-\frac{x_2^2}{2}\right)+\exp\left(-\frac{(x_1+4)^2}{2}-\frac{x_2^2}{2}\right)\,,
\]
and the initial to be $q^0(x)\propto \exp(-|x_1|^2/2-|x_2|^2/2)$. In the experiment, we choose $R_1=\frac{3\sqrt{5}}{10}$, $h=\eta=0.1$, $R_2=1.5$, $M_f=20$, and $N^\ast=10^3$. In Figure~\ref{Figure4}-\ref{Figure5}, we plot the samples generated by CEnLMC, LMC, and MALA at different iterations, using $N=10^4$. Since the example is not logconcave anymore, the convergence rate of the samples is slower. We also plot the ratio $\mathcal{R}_m$ at different iteration, using $N=2\times10^3, 6\times10^3$ and $10^4$ respectively, in Figure~\ref{Figure6}. While in the case of $N=2\times 10^3$, most particles need to have its gradient computed in every iteration, the ratio drops significantly for the larger $N$.
\begin{figure}[!ht]
     \centering
     \includegraphics[width=1\textwidth,height=0.4\textheight]{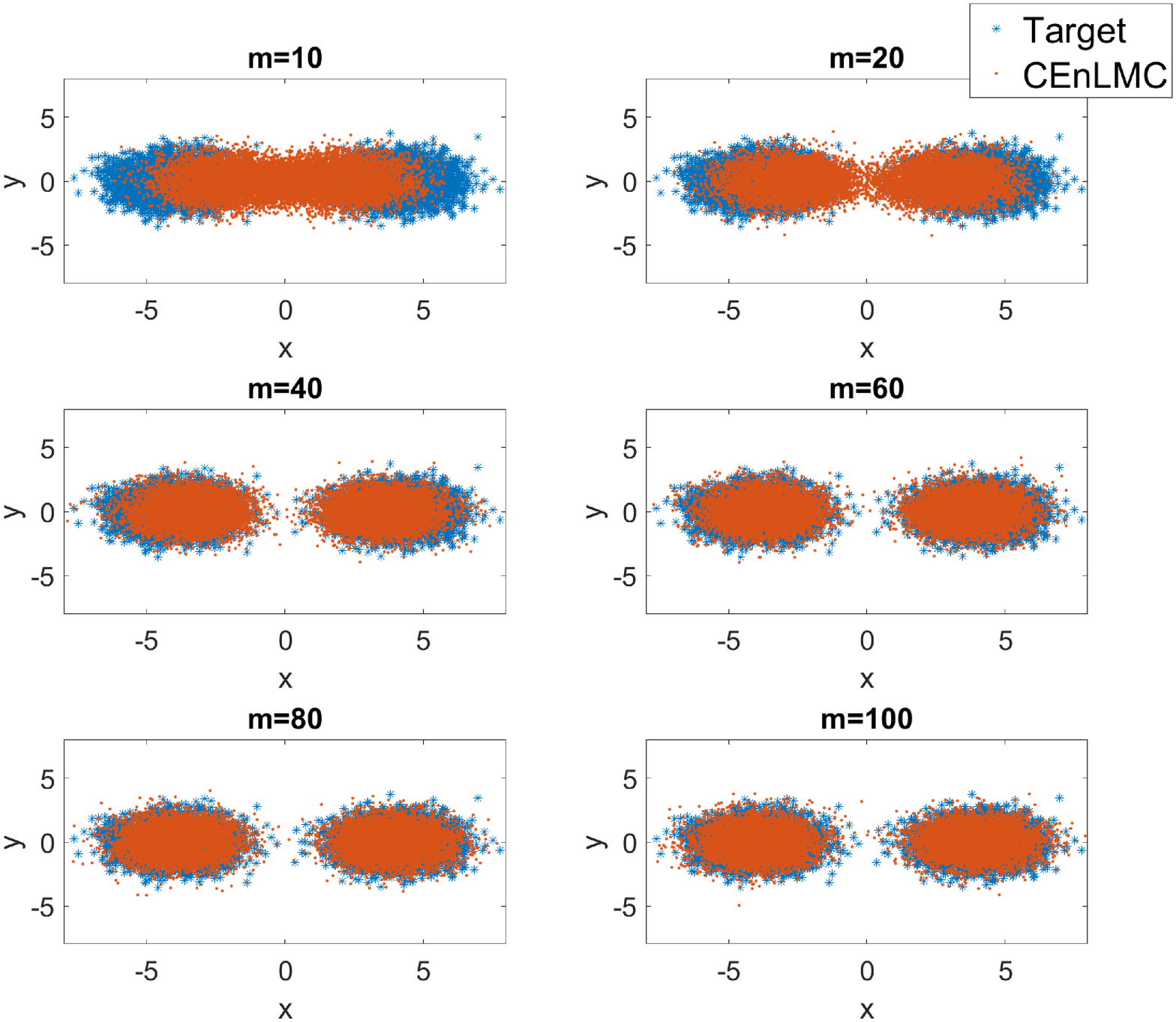}
     \caption{Example $2$: Evolution of samples using CEnLMC when $N=10^4$}
     \label{Figure4}
\end{figure}
\begin{figure}[!ht]
     \centering
     \includegraphics[width=1\textwidth,height=0.4\textheight]{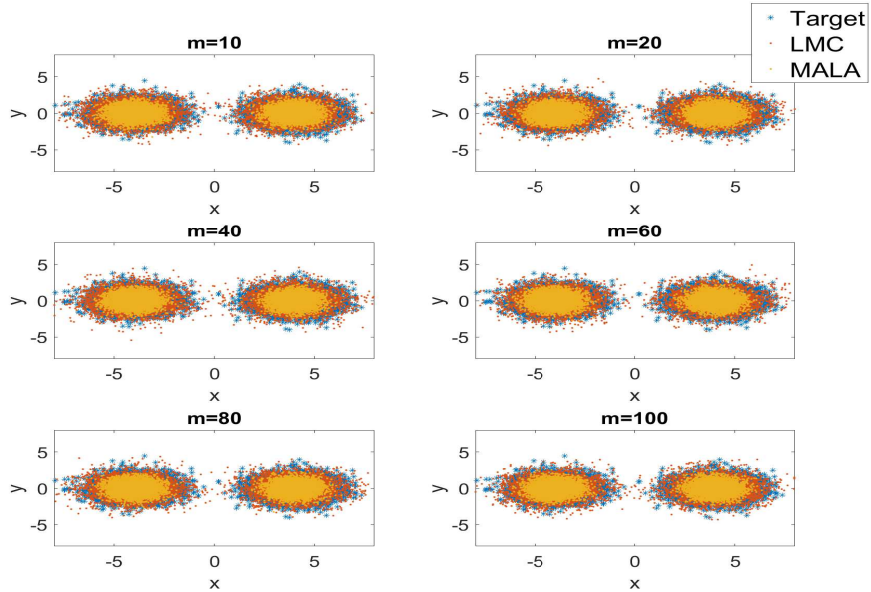}
     \caption{Example $2$: Evolution of samples using LMC and MALA when $N=10^4$}
     \label{Figure5}
\end{figure}
\begin{figure}[!ht]
     \centering
     \includegraphics[width=1\textwidth,height=0.4\textheight]{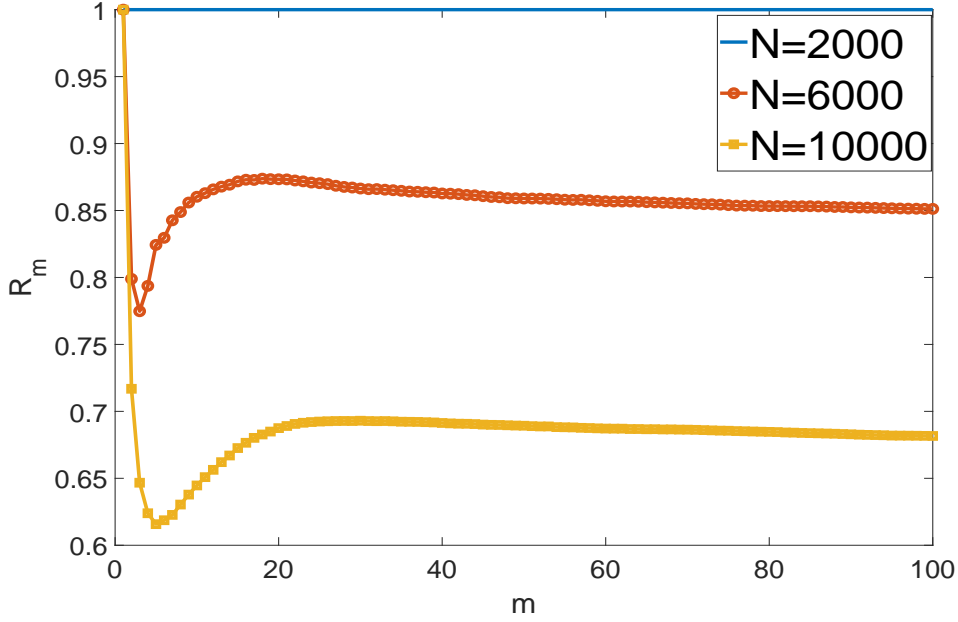}
     \caption{Example $2$: Evolution of $\mathcal{R}_m$ with $m$ when $N=2\times10^3, 6\times10^3, 10^4$}
     \label{Figure6}
\end{figure}
}

\section{Proof of theoretical results}\label{sec:proof}
\subsection{Proof of Theorem \ref{thm:blowupvariance}}\label{sec:proof_ENLMC}
In this section, we prove Theorem \ref{thm:blowupvariance}. According to algorithm \ref{alg:EnOLMC}, we have
\[
x^{m}_i=x^{m-1}_i-hF^{m-1}_i+\sqrt{2h}\xi^{m-1}_i
\]
and $\{\xi^{m-1}_i\}^N_{i=1}$ are i.i.d. independent. Under filtration $\mathcal{F}^{m-1}$, then the conditional distribution of $\left\{x^{m}_i\right\}^N_{i=1}$ is independent.

To prove the theorem, we need the following proposition:
\begin{proposition}\label{prop:blowupvariance}
Assume $\left\{x^m_i\right\}^N_{i=1}$ are generated from Algorithm \ref{alg:EnOLMC} with $F^m$ defined as \eqref{eqn:secondapproximation}, then for $f(x) = x^2/2$, we have: for any $m>0$, $1\leq i\leq N$
\begin{equation}\label{varianceinfinity}
\mathbb{E}\left(\left|E^m_i\right|^2\right)=\mathbb{E}\left|F^m_i-\nabla f(x^m_i)\right|^2=\infty\,.
\end{equation}
\end{proposition}

\begin{proof}[Proof of Proposition \ref{prop:blowupvariance}]
Since $f(x)=|x|^2/2$, we can obtain, according to~\eqref{eqn:F^m_ij_formula}:
\[
\begin{aligned}
F^m_{i,j}=&\frac{1}{\eta}\frac{(x^{m}_j+x^{m}_i)(x^m_j-x^m_i)}{2|x^m_j-x^m_i|^2}\frac{\mathbf{1}_{|\delta x^m_{ij}|<\eta}}{p^{m}_j(x^m_j)}(x^m_j-x^m_i)\\
=&\frac{x^m_j-x^m_i}{2\eta}\frac{\mathbf{1}_{|\delta x^m_{ij}|<\eta}}{p^{m}_j(x^m_j)}+\frac{x^m_i}{\eta}\frac{\mathbf{1}_{|\delta x^m_{ij}|<\eta}}{p^{m}_j(x^m_j)}\,.
\end{aligned}
\]

The two terms carry different information:
\begin{itemize}
\item The conditional expectation of first term equals zero:
\[
\begin{aligned}
&\mathbb{E}\left(\frac{x^m_j-x^m_i}{2\eta}\frac{\mathbf{1}_{|\delta x^m_{ij}|<\eta}}{p^{m}_j(x^m_j)}\middle|\mathcal{F}^{m-1}\right)\\
=&\frac{1}{2\eta}\int\int_{|x^m_j-x^m_i|<\eta} (x^m_j-x^m_i)p^{m}_i(x^m_i)\rd x^{m}_j\rd x^{m}_i=0\,.
\end{aligned}
\]

\item The second term is consistent with $\nabla f(x^m_i)=x^m_i$, meaning:
\[
\mathbb{E}\left(\frac{x^m_i}{\eta}\frac{\mathbf{1}_{|\delta x^m_{ij}|<\eta}}{p^{m}_j(x^m_i)}\middle|\mathcal{F}^{m-1},x^m_i\right)=x^m_i\int_{|x^m_j-x^m_i|<\eta} \frac{1}{\eta}\rd x^{m}_j=x^m_i\,,
\]
where we use $x^m_j$ and $x^m_i$ is conditional independent in the first equality.
\end{itemize}

These imply, for all $j\neq i$:
\begin{equation}\label{eqn:zero1}
\mathbb{E}\left(F^m_{i,j}-x^m_i\middle| \mathcal{F}^{m-1}\right)=0\,.
\end{equation}
Furthermore, since the conditional distribution of $x^m_{j_1},x^m_{j_2},x^m_i$ are independent, for $j_1\neq j_2$, $i\neq j_1$, and $i\neq j_2$:
\begin{equation}\label{eqn:zero2}
\begin{aligned}
&\mathbb{E}\left((F^m_{i,j_1}-x^m_i)(F^m_{i,j_2}-x^m_i)\middle| \mathcal{F}^{m-1}\right)\\
=&\mathbb{E}\left(\mathbb{E}\left((F^m_{i,j_1}-x^m_i)(F^m_{i,j_2}-x^m_i)\middle| \mathcal{F}^{m-1},x^m_i\right)\middle| \mathcal{F}^{m-1}\right)\\
=&\mathbb{E}\left(\mathbb{E}\left(F^m_{i,j_1}-x^m_i\middle| \mathcal{F}^{m-1},x^m_i\right)\mathbb{E}\left(F^m_{i,j_2}-x^m_i\middle| \mathcal{F}^{m-1},x^m_i\right)\middle| \mathcal{F}^{m-1}\right)\\
=&0
\end{aligned}
\end{equation}

Plug \eqref{eqn:zero1} and~\eqref{eqn:zero2} into $\mathbb{E}\left(\left|E^m_i\right|^2\middle|\mathcal{F}^{m-1}\right)=\mathbb{E}\left(\left|F^m_i-\nabla f(x^m_i)\right|^2\middle|\mathcal{F}^{m-1}\right)$, we have
\begin{equation}\label{eqn:expansionfirst}
\begin{aligned}
\mathbb{E}\left(\left|E^m_i\right|^2\middle|\mathcal{F}^{m-1}\right)=&\mathbb{E}\left(\left|F^m_i-\nabla f(x^m_i)\right|^2\middle|\mathcal{F}^{m-1}\right)\\
=&\frac{1}{(N-1)^2}\sum^N_{j\neq i}\mathbb{E}\left(\left|F^m_{i,j}-x^m_i\right|^2\middle|\mathcal{F}^{m-1}\right)\\
=&\frac{1}{(N-1)^2}\sum^N_{j\neq i}\mathbb{E}\left(\left|F^m_{i,j}\right|^2\middle|\mathcal{F}^{m-1}\right)-\frac{1}{N-1}\mathbb{E}\left(\left|x^{m}_i\right|^2\middle|\mathcal{F}^{m-1}\right)\,,
\end{aligned}
\end{equation}
where we use \eqref{eqn:zero2} in the second equality.
Noting that in~\eqref{eqn:var_F_comp2} we already showed:
\[
\mathbb{E}\left(\left|F^m_{i,j}\right|^2\middle|\mathcal{F}^{m-1}\right)=\infty\,,
\]
and that the second term in~\eqref{eqn:expansionfirst} is finite:
\[
\mathbb{E}\left(\left|x^{m}_i\right|^2\middle|\mathcal{F}^{m-1}\right)=\left|x^{m-1}_i-hF^{m-1}_i\right|^2+2h<\infty\,,
\] 
we obtain:
\[
\mathbb{E}\left(\left|F^m_i-\nabla f(x^m_i)\right|^2\middle|\mathcal{F}^{m-1}\right)=\infty\,,
\]
which proves \eqref{varianceinfinity}, concluding this proposition.
\end{proof}

Now, we are ready to prove Theorem \ref{thm:blowupvariance}.
\begin{proof}[Proof of Theorem \ref{thm:blowupvariance}]
For each $m\geq 0$ and $1\leq i\leq N$, we consider 
\[
x^{m+1}_i=x^{m}_i-h\nabla f(x^{m}_i)+\sqrt{2h}\xi^{m}_i+hE^m_i\,,
\]
where $E^{m}_i=\nabla f(x^{m}_i)-F^{m}_i$ denote the differentiation from the classical LMC formula. Using $x^m_j$ and $x^m_i$ are conditional independent for $i\neq j$, we obtain
\begin{equation}\label{eqn:zero3}
\begin{aligned}
&\mathbb{E}\left(E^{m}_i(x^{m}_i-h\nabla f(x^{m}_i)+\sqrt{2h}\xi^{m}_i)\middle|\mathcal{F}^{m-1}\right)\\
=&\mathbb{E}\left(E^{m}_i(x^{m}_i-h\nabla f(x^{m}_i))\middle|\mathcal{F}^{m-1}\right)\\
=&\mathbb{E}\left(\mathbb{E}\left(E^{m}_i(x^{m}_i-h\nabla f(x^{m}_i))\middle|\mathcal{F}^{m-1},x^m_i\right)\middle|\mathcal{F}^{m-1}\right)\\
=&\mathbb{E}\left(\left(\frac{1}{N-1}\sum^N_{j\neq i}\mathbb{E}\left(x^m_i-F^{m}_{i,j}\middle|\mathcal{F}^{m-1},x^m_i\right)\right)\left(x^{m}_i-h\nabla f(x^{m}_i)\right)\middle|\mathcal{F}^{m-1}\right)\\
=&\mathbb{E}\left(0\left(x^{m}_i-h\nabla f(x^{m}_i)\right)\middle|\mathcal{F}^{m-1}\right)=0\,,
\end{aligned}
\end{equation}
where we use $\mathbb{E}\left(\xi^{m}_i\middle|\mathcal{F}^{m-1}\right)=\mathbb{E}\left(\xi^{m}_i\right)=\vec{0}$ in the first equality and \eqref{eqn:zero1} in the second last equality.

Therefore, we have
\[
\begin{aligned}
&\mathbb{E}\left(|x^{m+1}_i|^2\middle|\mathcal{F}^{m-1}\right)\\
=&\mathbb{E}\left(\left|x^{m}_i-h\nabla f(x^{m}_i)+\sqrt{2h}\xi^{m}_i\right|^2\middle|\mathcal{F}^{m-1}\right)+\mathbb{E}\left(|E^{m}_i|^2\middle|\mathcal{F}^{m-1}\right)\\
\geq&\mathbb{E}\left(|E^{m}_i|^2\middle|\mathcal{F}^{m-1}\right)\,,
\end{aligned}
\]
where we use \eqref{eqn:zero3} in the first equality. Finally, using the previous proposition, we have
\[
\mathbb{E}\left(\mathbb{E}\left(|x^{m+1}_i|^2\middle|\mathcal{F}^{m-1}\right)\right)\geq \mathbb{E}\left(\mathbb{E}\left(|E^{m}_i|^2\middle|\mathcal{F}^{m-1}\right)\right)=\infty\,,
\]
which proves \eqref{eqn:blowupvarianceinfinity}.
\end{proof}

\subsection{Analysis of CEnLMC}\label{pfthCENLMC}
We now analyze Algorithm~\ref{alg:CEnOLMC}, the Constraint Ensemble LMC. The strategy is to compare the evolution of $x^m_i$ with $z^m_i$, the solution to the classical LMC~\eqref{eqn:lmc}, before utilizing the convergence of $z^m_i$ to find the convergence of $x^m_i$.

Theorem~\ref{thm:convergenceofCEnLMC_z} discusses the closeness of $x^m_i$ and $z^m_i$, while Theorem~\ref{thm:convergenceofCEnLMC} discusses the convergence of $x^m_i$. The following two subsections are dedicated to these two theorems respectively.

\subsubsection{Proof of Theorem~\ref{thm:convergenceofCEnLMC_z}}\label{pfthm:disxz}
To show the smallness of $x^m_i-z^m_i$, we first rewrite the updating formula for $x^m_i$,~\eqref{eqn:algcelmc}, into
\begin{equation}\label{eqn:algcelmcmodification}
x^{m+1}_i=x^m_i-\nabla f(x^m_i)h+E^m_ih+\sqrt{2h}\xi^m_i\,,
\end{equation}
where
\begin{equation}\label{eqn:Em}
E^m_i=\nabla f(x^m_i)-F^m_i\,.
\end{equation}

Comparing the updating formula of $z^m_i$ in equation~\eqref{eqn:lmc}, it is easy to see that the key lies in bounding the term $E^m_i$. This is shown in the following lemma.
\begin{lemma}\label{lem:E^m_i}
Under the same conditions of Theorem \ref{thm:convergenceofCEnLMC_z}, we have: for any $m\geq0$, $1\leq i\leq N$
\begin{equation}\label{eqn:prop:bound}
\mathbb{E}\left|E^m_i\right|\lesssim \sqrt{\frac{R^{d}_1L\textcolor{black}{(M_f-f^*)}d^2}{\eta^{d}N^\ast}}\exp\left(\frac{R_2(R_2+R_1)}{2h}\right)+L\eta d\,.
\end{equation}
\end{lemma}

Theorem~\ref{thm:convergenceofCEnLMC_z} is a direct consequence from this lemma.
\begin{proof}[Proof of Theorem \ref{thm:convergenceofCEnLMC_z}]
For each $m\geq0$, $1\leq i\leq N$, we subtract~\eqref{eqn:algcelmcmodification} and~\eqref{eqn:lmc} to obtain
\begin{equation}\label{eqn:iteration}
\EE\left|x^{m+1}_i-z^{m+1}_i\right|=\EE\left|(x^m_i-z^m_i)-h(\nabla f(x^m_i)-\nabla f(z^m_i))\right|+h\EE|E^m_i|\,.
\end{equation}
Noting that $\nabla f$ is $L$-Lipschitz continuous,
\[
\left|\nabla f(x^m_i)-\nabla f(z^m_i))\right|\leq Lh\left|x^m_i-z^m_i\right|\,,
\]
then
\[
\left|(x^m_i-z^m_i)-h(\nabla f(x^m_i)-\nabla f(z^m_i))\right|\leq (1+Lh)\left|x^m_i-z^m_i\right|\,.
\]
We take the expectation, and utilize Lemma~\ref{lem:E^m_i}:
\[
\begin{aligned}
\EE\left|x^{m+1}_i-z^{m+1}_i\right|\leq &(1+Lh)\EE\left|x^m_i-z^m_i\right|\\
&+h\left(\sqrt{\frac{R^{d}_1L\textcolor{black}{(M_f-f^*)}d^2}{\eta^{d}N^\ast}}\exp\left(\frac{R_2(R_2+R_1)}{2h}\right)+L\eta d\right)\,.
\end{aligned}
\]
Use this formula iteratively, we have:
\[
\begin{aligned}
\EE\left|x^{m}_i-z^{m}_i\right|\leq &(1+Lh)^m\EE\left|x^m_0-z^m_0\right|\\
&+(1+Lh)^m\left(\sqrt{\frac{R^{d}_1\textcolor{black}{(M_f-f^*)}d^2}{L\eta^{d}N^\ast}}\exp\left(\frac{R_2(R_2+R_1)}{2h}\right)+\eta d\right)\,.
\end{aligned}
\]
Noting $x^m_0 = z^m_0$, the first term is eliminated, and we conclude~\eqref{eqn:thm:disxznonconvex}. When $f$ is $\mu$-convex,
\[
\nabla f(x^m_i)-\nabla f(z^m_i)\geq \mu (x^m_i-z^m_i)\,,
\]
then for $h$ small enough:
\[
\left|(x^m_i-z^m_i)-h(\nabla f(x^m_i)-\nabla f(z^m_i))\right|\leq (1-\mu h)\left|x^m_i-z^m_i\right|\,.
\]
Running the same argument as above, and relaxing $(1-\mu h)^m\leq 1$, we conclude~\eqref{eqn:thm:disxz}.
\end{proof}

We now prove Lemma~\ref{lem:E^m_i}
\begin{proof}[Proof of Lemma \ref{lem:E^m_i}]
We first define:
\begin{equation}\label{eqn:Gmibetter}
G^m_i=\left\{
\begin{aligned}
&\nabla f(x^m_i),\quad & \sqrt{2h}|\xi^{m-1}_i|>R_1\;\text{or}\; \textcolor{black}{f(x^m_i)> M_f}\;\text{or}\; N^\ast>N^m_i\\
&\frac{1}{N^m_i}\sum^N_{j\neq i}G^m_{ij}\,,\quad & \text{otherwise}\,.
\end{aligned}
\right.
\end{equation}
where
\[
G^m_{ij}=\alpha_d\frac{\langle \nabla f(x^m_i),\delta x^m_{ij}\rangle}{|\delta x^m_{ij}|^2}\frac{\mathbf{1}_{|\delta x^m_{ij}|\leq\eta\,,|\delta w^m_{ij}|\leq R_2}}{\mathrm{p}^m_j}\delta x^m_{ij}
\]
is the counterpart of $F^m_{ij}$ that eliminates the discretization error. Then
\[
|E^m_i|=|\nabla f(x^m_i)-F^m_i| \leq |\nabla f(x^m_i)-G^m_i| +|G^m_i - F^m_i|\,.
\]
Clearly the term $ |\nabla f(x^m_i)-G^m_i|$ is the ensemble error and the term $|G^m_i - F^m_i|$ takes care of the discretization error.

To control $|G^m_i - F^m_i|$, we define
\[
\mathbf{1}_{\Omega_i}=\mathbf{1}_{|N^{m}_i|\geq N^\ast}\mathbf{1}_{\textcolor{black}{f(x^m_i)\leq M_f}}\mathbf{1}_{\sqrt{2h}|\xi^{m-1}_i|\leq R_1}\,,
\]
then
\begin{equation}\label{bound:discrete1}
\begin{aligned}
\EE\left(\left|G^m_{i}-F^m_{i}\right|\middle|\mathcal{F}^{m-1}\right) &= \EE\left(\mathbf{1}_{\Omega_i}\left|G^m_{i}-F^m_{i}\right|\middle|\mathcal{F}^{m-1}\right)\\
&\leq \EE\left(\frac{\mathbf{1}_{\Omega_i}}{N^m_i}\sum^N_{j\neq i}\left|G^m_{i,j}-F^m_{i,j}\right|\middle|\mathcal{F}^{m-1}\right)\\
&=\frac{1}{N^m_i}\sum^N_{j\neq i}\EE\left(\mathbf{1}_{\Omega_i}\left|G^m_{i,j}-F^m_{i,j}\right|\middle|\mathcal{F}^{m-1}\right)\\
&\leq \max_{1\leq j\leq N}\EE\left(\left|G^m_{i,j}-F^m_{i,j}\right|\middle|\mathcal{F}^{m-1}\right)\,.
\end{aligned}
\end{equation}
Plugging \eqref{eqn:error_GF} into \eqref{bound:discrete1}, we obtain
\begin{equation}\label{bound:discrete3}
\begin{aligned}
\EE\left(\left|G^m_{i}-F^m_{i}\right|\right)=\EE\left(\EE\left(\left|G^m_{i}-F^m_{i}\right|\middle|\mathcal{F}^{m-1}\right)\right)\leq L\eta d\,.
\end{aligned}
\end{equation}

To control $\left|G^m_{i}-\nabla f(x^m_i)\right|$. We note
\begin{equation}
\mathbb{E}\left(\left|G^m_{i}-\nabla f(x^m_i)\right|^2\right)=\mathbb{E}\left(\EE\left(\mathbf{1}_{\Omega_i}\left|G^m_{i}-\nabla f(x^m_i)\right|^2\middle|\mathcal{F}^{m-1}\right)\right)\,.
\end{equation}
Define 
\[
\mathcal{E}^m_{i,j}=G^m_{i,j}-\nabla f(x^m_i)\mathbf{1}_{|\delta w^m_{ij}|<R_2}\,,
\]
then
\begin{equation}\label{eqn:conerror1}
\begin{aligned}
&\mathbb{E}\left(\left|G^m_{i}-\nabla f(x^m_i)\right|^2\right)\\
=&\EE\left(\EE\left(\mathbf{1}_{\Omega_i}\left|\frac{1}{N^m_i}\sum_{j\neq i}\left[G^m_{ij}-\nabla f(x^m_i)\mathbf{1}_{|\delta w^m_{ij}|<R_2}\right]\right|^2\middle|\mathcal{F}^{m-1}\right)\right)\\
\leq& \EE\left(\EE\left(\frac{\mathbf{1}_{\Omega_i}}{(N^{m}_i)^2}\left|\sum_{j\neq i}G^m_{ij}-\nabla f(x^m_i)\mathbf{1}_{|\delta w^m_{ij}|<R_2}\right|^2\middle|\mathcal{F}^{m-1}\right)\right)\\
= &\EE\left(\EE\left(\frac{\mathbf{1}_{\Omega_i}}{(N^{m}_i)^2}\left| \sum_{j\neq i}\mathcal{E}^m_{i,j}\right|^2\middle|\mathcal{F}^{m-1}\right)\right)\\
\leq&\frac{1}{N^\ast}\EE\left(\left\{\max_{j}\EE\left(\mathbf{1}_{\Omega_i}\left|\mathcal{E}^m_{i,j}\right|^2\middle|\mathcal{F}^{m-1}\right)+\sum^N_{j_1\neq j_2}\EE\left(\mathbf{1}_{\Omega_i}\left\langle\mathcal{E}^m_{i,j_1},\mathcal{E}^m_{i,j_2}\right\rangle\middle|\mathcal{F}^{m-1}\right)\right\}\right)\\
=&\frac{1}{N^\ast}\EE\left(\max_{j}\EE\left(\mathbf{1}_{\Omega_i}\left|\mathcal{E}^m_{i,j}\right|^2\middle|\mathcal{F}^{m-1}\right)\right)\,,
\end{aligned}
\end{equation}
where we use $N^m_i=\sum^N_{j\neq i}\mathbf{1}_{|\delta w^m_{ij}|<R_2}$ in the first equality.

In the last equation, we note that
\begin{equation}\label{eqn:consistence}
\EE\left(\mathbf{1}_{\Omega_i}\mathcal{E}^m_{i,j}\middle|\mathcal{F}^{m-1},x^m_i\right)=\mathbf{1}_{\Omega_i}\EE\left(\mathcal{E}^m_{i,j}\middle|\mathcal{F}^{m-1},x^m_i\right)=0\,,
\end{equation}
with the conditional independence, and thus
\[
\begin{aligned}
&\EE\left(\mathbf{1}_{\Omega_i}\left\langle\mathcal{E}^m_{i,j_1},\mathcal{E}^m_{i,j_2}\right\rangle\middle|\mathcal{F}^{m-1}\right)\\
=&\EE\left(\EE\left(\mathbf{1}_{\Omega_i}\left\langle\mathcal{E}^m_{i,j_1},\mathcal{E}^m_{i,j_2}\right\rangle\middle|\mathcal{F}^{m-1},x^m_i\right)\middle|\mathcal{F}^{m-1}\right)\\
=&\EE\left(\left\langle\EE\left(\mathbf{1}_{\Omega_i}\mathcal{E}^m_{i,j_1}\middle|\mathcal{F}^{m-1},x^m_i\right),\EE\left(\mathbf{1}_{\Omega_i}\mathcal{E}^m_{i,j_2}\middle|\mathcal{F}^{m-1},x^m_i\right)\right\rangle\middle|\mathcal{F}^{m-1}\right)\\
=&0
\end{aligned}
\]

To further control~\eqref{eqn:conerror1} we simply use the direct calculation: for any $j\neq i$
\begin{equation}\label{eqn:firsttermbound}
\begin{aligned}
&\EE\left(\mathbf{1}_{\Omega_i}\left|\mathcal{E}^m_{i,j}\right|^2\middle|\mathcal{F}^{m-1}\right)\leq \EE\left(\mathbf{1}_{\Omega_i}\left|G^m_{i,j}\right|^2\middle|\mathcal{F}^{m-1}\right)\\
\leq& \alpha_d^2\mathbf{1}_{|\delta w^m_{ij}|<R_2}\int_{B(w^{m}_i,R_1)}\int_{B(x^m_i,\eta)}\frac{|\nabla f(x^m_i)|^2}{p^{m}_j(x^m_j)}p^{m}_i(x^m_i)\rd x^m_j \rd x^m_i\\
\stackrel{(\mathrm{I})}{\lesssim}& L\textcolor{black}{(M_f-f^*)}\alpha_d^2\int_{B(w^{m}_i,R_1)}\int_{B(x^m_i,\eta)}\mathbf{1}_{|\delta w^m_{ij}|<R_2}\exp\left(\frac{|x^m_j-w^{m}_j|^2}{4h}-\frac{|x^m_i-w^{m}_i|^2}{4h}\right)\rd x^m_j\rd x^m_i\\
\stackrel{(\mathrm{II})}{\lesssim}& L\textcolor{black}{(M_f-f^*)}\alpha_d^2\int_{B(w^{m}_i,R_1)}\int_{B(0,\eta)}\mathbf{1}_{|\delta w^m_{ij}|<R_2}\exp\left(\frac{|y+z-w^{m}_j|^2}{4h}-\frac{|y-w^{m}_i|^2}{4h}\right)\rd z\rd y\\
\stackrel{(\mathrm{III})}{\lesssim}&L\textcolor{black}{(M_f-f^*)}\alpha_d^2\exp\left(\frac{\eta^2+2(\eta R_1+\eta R_2+ R_2R_1)+R^2_2}{4h}\right)\int_{B(w^{m}_i,R_1)}\int_{B(0,\eta)}\rd z\rd y\\
=&\frac{R^{d}_1d^2L\textcolor{black}{(M_f-f^*)}}{\eta^d}\exp\left(\frac{\eta^2+2(\eta R_1+\eta R_2+ R_2R_1)+R^2_2}{4h}\right)\,.
\end{aligned}
\end{equation}
Here in $(\mathrm{I})$ we used $\frac{1}{2L}|\nabla f(x^m_i)|^2\leq f(x^m_i)-f^\ast<\textcolor{black}{(M_f-f^*)}$, in $(\mathrm{II})$ we used change of variables $y=x^m_i,z=x^m_j-x^m_i$. In $(\mathrm{III})$, we used:
\[
\begin{aligned}
&\exp\left(\frac{|y+z-w^{m}_j|^2}{4h}-\frac{|y-w^{m}_i|^2}{4h}\right)\\
=&\exp\left(\frac{|y-w^m_i+z+w^m_i-w^{m}_j|^2}{4h}-\frac{|y-w^{m}_i|^2}{4h}\right)\\
=&\exp\left(\frac{|z+w^m_i-w^{m}_j|^2}{4h}+\frac{\left\langle y-w^{m}_i,z+w^m_i-w^{m}_j\right\rangle}{2h}\right)\\
\lesssim&\exp\left(\frac{|z|^2}{4h}+\frac{|w^m_i-w^{m}_j|^2}{4h}+\frac{|z||w^m_i-w^{m}_j|}{2h}+\frac{|y-w^{m}_i|\left(|z|+|w^m_i-w^{m}_j|\right)}{2h}\right)\,.
\end{aligned}
\]

Plug \eqref{eqn:firsttermbound} into \eqref{eqn:conerror1}, we have
\begin{equation}\label{eqn:conerror2}
\mathbb{E}\left(\left|G^m_{i}-\nabla f(x^m_i)\right|^2\right)\lesssim \frac{R^{d}_1d^2L\textcolor{black}{(M_f-f^*)}}{N^\ast\eta^d}\exp\left(\frac{\eta^2+2(\eta R_1+\eta R_2+ R_2R_1)+R^2_2}{4h}\right)\,.
\end{equation}
Using $\eta<R_2$ and H\"older inequality we have
\[
\begin{aligned}
\mathbb{E}\left(\left|G^m_{i}-\nabla f(x^m_i)\right|\right)&=\left(\mathbb{E}\left(\left|G^m_{i}-\nabla f(x^m_i)\right|^2\right)\right)^{1/2}\\
&\lesssim \sqrt{\frac{R^{d}_1d^2L\textcolor{black}{(M_f-f^*)}}{N^\ast\eta^d}}\exp\left(\frac{R_2(R_2+R_1)}{2h}\right)\,.
\end{aligned}
\]
Combine it with~\eqref{bound:discrete3} we prove \eqref{eqn:prop:bound}.
\end{proof}

\subsubsection{Proof of Theorem~\ref{thm:convergenceofCEnLMC}}
The validity of Theorem~\ref{thm:convergenceofCEnLMC} is built upon the fact that $x_i^m$ system and $z^m_i$ system are close, shown above, and that the $z^m_i$ system follows LMC, which converges to the target distribution.

It is a classical result to show that the LMC solution converges. To do so, one constructs another particle system that is drawn from the target distribution. Let $y_0$ be a random vector drawn from target distribution induced by $p$, and set
\begin{equation}\label{eqn:yt}
y_i(t)=y^0_i-\int^t_0 \nabla f(y_i(s))\rd s+\sqrt{2}\int^t_0\rd B_i(s)\,,
\end{equation}
where we construct Brownian motion that satisfies:
\begin{equation}\label{eqn:Bt}
B_i(h(m+1))-B_i(hm)=\sqrt{h}\xi^m_i\,.
\end{equation}
Then $y_i(t)$ is drawn from the distribution induced by $p$ as well. On the discrete level, let $y^m_i=y_i(hm)$, then:
\begin{equation}\label{eqn:ymolmc}
y^{m+1}_i=y^{m}_i-\int^{(m+1)h}_{mh} \nabla f(y_i(s))\rd s+\sqrt{2h}\xi^{m}_i\,.
\end{equation}

Since $y^m_i\sim p(x)$, then we have
\[
W_1(q^m_i,p)\leq \mathbb{E}|x^m_i-y^m_i|\,,
\]
where $\mathbb{E}$ takes all randomness into account. Choose the initial data $y_0$ so that $W_1(q^0,p)=\EE|x^0_i-y^0_i|$. Then the problem boils down to showing that $x^m_i$ is close to $y^m_i$. Since we already know that $x^m_i$ and $z^m_i$ are close, we now need to show the closeness between $z$ and $y$. This classical result regarding the convergence of LMC was shown in~\cite{pmlr-v80-chatterji18a,DALALYAN20195278}, and we cite it here for the completeness of the paper (with notations adjusted to our setting).

\begin{proposition}[Closeness of $z$ and $y$]\label{thm:disyz}
Assume conditions of Theorem \ref{thm:convergenceofCEnLMC_z}, and let $f$ be $L$-smooth and $\mu$ convex with $\kappa=L/\mu$, we have: for any $m\geq0$, $1\leq i\leq N$
\begin{equation}\label{eqn:thm:disyz}
\mathbb{E}|z^m_i-y^m_i|\leq \exp\left(-\frac{\mu h m}{2}\right)W_{1}(q^0,p)+\mathcal{O}\left(\kappa \sqrt{hd}\right)\,.
\end{equation}
\end{proposition}
We leave the proof to Appendix \ref{pfthm:disyz}. We should emphasize that this result is essentially the same as the one in~\cite{DALALYAN20195278,durmus2017,dalalyan2018sampling}. The only difference is that we use $L_1$ norm for bounding $z^m_i-y^m_i$ for the consistency with the result in Theorem~\ref{thm:convergenceofCEnLMC_z}.

Now, we are ready to prove Theorem \ref{thm:convergenceofCEnLMC}.

\begin{proof}[Proof of Theorem \ref{thm:convergenceofCEnLMC}]

Combining Theorem~\ref{thm:convergenceofCEnLMC_z} and Proposition~\ref{thm:disyz} by adding \eqref{eqn:thm:disxz} and \eqref{eqn:thm:disyz} through the triangle inequality, we obtain
\begin{equation}\label{eqn:thm:disxy}
\begin{aligned}
\mathbb{E}|x^m_i-y^m_i|\leq &\mathbb{E}|x^m_i-z^m_i|+\mathbb{E}|z^m_i-y^m_i|\\
=&\exp\left(-\frac{\mu h m}{2}\right)W_{1}(q^0,p)\\
&+\mathcal{O}\left(\kappa (\sqrt{hd}+\eta d)+\sqrt{\frac{R^{d}_1\kappa d^2\textcolor{black}{(M_f-f^*)}}{\mu\eta^{d}N^\ast}}\exp\left(\frac{R_2(R_2+R_1)}{2h}\right)\right)\,.
\end{aligned}
\end{equation}
Since $W_1(q^m_i,p)\leq \mathbb{E}|x^m_i-y^m_i|$, we prove \eqref{eqn:convergenceofCEnLMC}. To prove \eqref{eqn:convergenceofCEnLMC2}, we use
\begin{equation}\label{eqn:thm:disxy1}
\EE\left|\frac{1}{N}\sum^N_{i=1}g(x^m_i)-\EE_{p}(g)\right|\leq \frac{1}{N}\sum^N_{i=1}\EE\left|g(x^m_i)-g(y^m_i)\right|+\EE\left|\frac{1}{N}\sum^N_{i=1}g(y^m_i)-\EE_{p}(g)\right|\,.
\end{equation}
Using the Lipschitz continuity, the first term is easily controlled.
\begin{equation}\label{eqn:thm:disxy2}
\frac{1}{N}\sum^N_{i=1}\EE\left|g(x^m_i)-g(y^m_i)\right|\leq \mathcal{O}\left(\frac{1}{N}\sum^N_{i=1}\EE|x^m_i-y^m_i|\right)\,.
\end{equation}
Here the $\mathcal{O}$ notation includes the Lipschitz constant of $g$. The second term of \eqref{eqn:thm:disxy1} is a standard central limit theorem:
\begin{equation}\label{eqn:thm:disxy3}
\begin{aligned}
&\EE\left|\frac{1}{N}\sum^N_{i=1}g(y^m_i)-\EE_{p}(g)\right|\leq \left(\EE\left(\frac{1}{N}\sum^N_{i=1}g(y^m_i)-\EE_{p}(g)\right)^2\right)^{1/2}\\
\leq &\left(\frac{1}{N^2}\sum^N_{i=1} \EE\left(g(y^m_i)-\EE_{p}(g)\right)^2\right)^{1/2}\leq \mathcal{O}\left(\frac{1}{\sqrt{N}}\right)\,.
\end{aligned}
\end{equation}
Combining~\eqref{eqn:thm:disxy},~\eqref{eqn:thm:disxy2} and \eqref{eqn:thm:disxy3} into \eqref{eqn:thm:disxy1}, we prove the weak convergence~\eqref{eqn:convergenceofCEnLMC2}.
\end{proof}

\subsection{Proof of Theorem \ref{thm:ratioofcalculation}}\label{pfthm:ratioofcalculation}
We prove Theorem \ref{thm:ratioofcalculation} in this section. First, we give another iteration lemma:
\begin{lemma}\label{lem:largeN2} Under conditions of Theorem \ref{thm:convergenceofCEnLMC_z}, let $m\geq 0$, and $\epsilon_m>0$. Then, there exists a constant $N'$ that is independent of $\eta,\epsilon_m$ such that if 
\[
N>N',\quad \EE|x^m_i-z^m_i|\leq \epsilon_m\,,\quad \forall 1\leq i\leq N
\]
we have
\begin{equation}\label{eqn:iterationboundnew2}
\EE|x^{m+1}_i-z^{m+1}_i|\leq \epsilon_m+\frac{B(\epsilon_m)}{\eta^{d/2}}+C\eta,\quad \mathbb{P}\left(N^m_i\leq N^*\right)\leq 1-B(\epsilon_m)\,,\quad \forall 1\leq i\leq N\,.
\end{equation}
where $C$ is a constant and $B:\mathbb{R}\rightarrow\mathbb{R}^+$ is a continuous function that satisfies
\[
\lim_{\epsilon_m\rightarrow0}B(\epsilon_m)=0\,.
\]
\end{lemma}
\begin{remark} We note that in Lemma \ref{lem:largeN2}, the constants $N'$, $C$ and function $B$ depend on other parameters such as $h,d,R_2,R_1,M_f,N^*,\mu,L$.
\end{remark}
\begin{proof}[Proof of Lemma \ref{lem:largeN2}]
Without loss of generality, we only consider $|x^m_1-z^m_1|$ and $N^m_1$. Similar to the argument in Lemma \ref{lem:E^m_i}, 
\[
\mathbb{E}\left|E^m_1\right|\lesssim \sqrt{\frac{R^{d}_1L\textcolor{black}{(M_f-f^*)}d^2}{\eta^{d}}}\exp\left(\frac{R_2(R_2+R_1)}{2h}\right)\EE\left(\frac{1}{\sqrt{N^m_1}}\right)+L\eta d\,.
\]
According to the proof of Theorem \ref{thm:convergenceofCEnLMC_z}, we obtain
\begin{equation}\label{eqn:iterationboundnew}
\begin{aligned}
\EE\left|x^{m+1}_1-z^{m+1}_1\right|\leq &(1-\mu h)\EE\left|x^m_1-z^m_1\right|\\
&+h\left(\sqrt{\frac{R^{d}_1L\textcolor{black}{(M_f-f^*)}d^2}{\eta^{d}}}\exp\left(\frac{R_2(R_2+R_1)}{2h}\right)\EE\left(\frac{1}{\sqrt{N^m_1}}\right)+L\eta d\right)\\
\leq &\epsilon_m+\frac{C}{\eta^{d/2}}\EE\left(\frac{1}{\sqrt{N^m_1}}\right)+C\eta
\,,
\end{aligned}
\end{equation}
where $C$ is a constant that is independent of $\eta$ and $\epsilon_m$. Thus, it suffices to bound $\EE\left(\frac{1}{\sqrt{N^m_1}}\right)$. Define
\[
\widetilde{w}^m_{i}=z^m_i-h\nabla f(z^m_i),\quad \widetilde{N}^m_1=\sum^{N_z}_{j>i}\mathbf{1}_{|\delta \widetilde{w}^m_{ij}|<R_2/4}\,, 
\]
where $N_z<N$ is a positive integer. According to \cite{NEURIPS2019_65a99bb7}, the KL divergence between the distribution of $z^m_i$ and target distribution is finite for all $m$. This implies the distribution of $z^m_i$ has a density. Thus, for any $M>0$, we have
\begin{equation}\label{eqn:probabilitybound0}
\lim_{N_z\rightarrow\infty}\mathbb{P}\left(\widetilde{N}^m_i>M\right)=1\,.
\end{equation}

Now, we start bounding $\EE\left(\frac{1}{\sqrt{N^m_1}}\right)$. Since $\EE|x^m_i-z^m_i|\leq \epsilon_m$, 
\[
\mathbb{P}\left(|x^m_i-z^m_i|>\frac{R_2}{4}\right)\leq \frac{4\epsilon_m}{R_2}\,,\quad \forall 1\leq i\leq N\,.
\]
which implies
\begin{equation}\label{eqn:probabilitybound}
\mathbb{P}\left(\cap^{N_z}_{i=1}\left\{|x^m_i-z^m_i|\leq \frac{R_2}{4}\right\}\right)\geq 1-\frac{4\epsilon_mN_z}{R_2}\,.
\end{equation}
According to the definition of $N^m_i$ \eqref{eqn:N_neighbors}, using \eqref{eqn:probabilitybound}, we obtain that for any $M<N_z$
\begin{equation}\label{eqn:probabilitybound2}
\mathbb{P}\left(N^m_i>M\right)\geq \mathbb{P}\left(\widetilde{N}^m_i>M\right)-\frac{4\epsilon_mN_z}{R_2}\,.
\end{equation}
From this,
\begin{equation}\label{eqn:probabilitybound3}
\EE\left(\frac{1}{\sqrt{N^m_1}}\right)\leq \frac{1}{\sqrt{M}}\left(\mathbb{P}\left(\widetilde{N}^m_i>M\right)-\frac{4\epsilon_mN_z}{R_2}\right)+\frac{1}{\sqrt{N^*}}\left[1-\left(\mathbb{P}\left(\widetilde{N}^m_i>M\right)-\frac{4\epsilon_mN_z}{R_2}\right)\right]\,.
\end{equation}
Define the right-side of \eqref{eqn:probabilitybound3} as $F(M,N_z,\epsilon_m)$. Since $M,N_z$ can be arbitrarily chosen, we have
\[
\EE\left(\frac{1}{\sqrt{N^m_1}}\right)\leq \inf_{M,N_z}F(M,N_z,\epsilon_m)\,
\]
Plugging this into \eqref{eqn:iterationboundnew}, 
\[
\EE\left|x^{m+1}_1-z^{m+1}_1\right|\leq \epsilon_m+\frac{C}{\eta^{d/2}}\inf_{M,N_z}F(M,N_z,\epsilon_m)+C\eta
\,.
\]
Noticing that
\begin{equation}\label{Flimit}
\lim_{M\rightarrow\infty}\lim_{N_z\rightarrow\infty}\lim_{\epsilon_m\rightarrow0}F(M,N_z,\epsilon_m)=0\,,
\end{equation}
we obtain the first inequality of \eqref{eqn:iterationboundnew2}. Next, for any $M>N^*$, because 
\[
\mathbb{P}\left(N^m_i>N^*\right)\geq \mathbb{P}\left(N^m_i>M\right)\geq \mathbb{P}\left(\widetilde{N}^m_i>M\right)-\frac{4\epsilon_mN_z}{R_2}\geq 1-\sqrt{N^*}F(M,N_z,\epsilon_m)\,,
\]
\eqref{Flimit} also implies the second inequality of \eqref{eqn:iterationboundnew2}.
\end{proof}
Now, we are ready to prove the theorem:
\begin{proof}[Proof of Theorem \ref{thm:ratioofcalculation}]
Noticing that when $m=0$,
\[
\EE|x^0_i-z^0_i|=0\,.
\]
Using Lemma \ref{lem:largeN2} \eqref{eqn:iterationboundnew2}, for any $\epsilon>0$, we have
\[
\lim_{\eta\rightarrow0}\lim_{N\rightarrow\infty}\EE|x^1_i-z^1_i|<\epsilon,\quad \lim_{\eta\rightarrow0}\lim_{N\rightarrow\infty}\mathbb{P}\left(\left\{N^1_i<N^\ast\right\}\right)>1-\epsilon\,.    
\]
Repeating this process with Lemma \ref{lem:largeN2}, we obtain 
\begin{equation}\label{eqn:Aconvergence}
\lim_{\eta\rightarrow0}\lim_{N\rightarrow\infty}\sup_{0\leq m\leq M,1\leq i\leq N}\EE|x^m_i-z^m_i|=0\,.\quad 
\end{equation}

Next, to prove \eqref{boundofP1}, we notice that for $m\geq0$ and $1\leq i\leq N$
\[
x^{m}_i-w^{m}_i=\sqrt{2h}\xi^{m-1}_i\,,
\]
which implies
\[
\begin{aligned}
&\mathbb{P}\left(\left\{|x^m_i-w^{m}_i|>R_1\right\}\right)=\mathbb{P}\left(\left\{|\xi^{m-1}_i|>\frac{R_1}{\sqrt{2h}}\right\}\right)\\
=&\int_{|x|>\frac{R_1}{\sqrt{2h}}}\frac{1}{(2\pi)^{d/2}}\exp\left(-\frac{|x|^2}{2}\right)\rd x=\frac{S_d}{(2\pi)^{d/2}}\int^\infty_{\frac{R_1}{\sqrt{2h}}}r^{d-1}\exp\left(-\frac{r^2}{2}\right)\rd r\\
\leq&\frac{S_d}{(2\pi)^{d/2}}\int^\infty_{\frac{R_1\sqrt{d}}{\sqrt{2}}}r^{d-1}\exp\left(-\frac{r^2}{2}\right)\rd r\,,
\end{aligned}
\]
where the last inequality comes from $h<\frac{1}{d}$.

Then, to prove \eqref{boundofP2}, we first use $f(x^m_i)-f^\ast\leq \frac{1}{2\mu}|\nabla f(x^m_i)|^2$ to obtain
\begin{equation}\label{eqn:boundofP20}
\begin{aligned}
&\mathbb{P}\left(\left\{f(x^m_i)-f^\ast>\textcolor{black}{(M_f-f^*)}\right\}\right)\\
=&\mathbb{P}\left(\left\{f(x^m_i)-f^\ast>\textcolor{black}{(M_f-f^*)}\right\}\right)
\leq\mathbb{P}\left(\left\{|\nabla f(x^m_i)|^2>2\mu \textcolor{black}{(M_f-f^*)}\right\}\right)\\
\leq&\mathbb{P}\left(\left\{|\nabla f(y^m_i)|^2+|\nabla f(x^m_i)-\nabla f(y^m_i)|^2>\mu \textcolor{black}{(M_f-f^*)}\right\}\right)\\
\leq &\mathbb{P}\left(\left\{|\nabla f(x^m_i)-\nabla f(y^m_i)|^2>\frac{\mu \textcolor{black}{(M_f-f^*)}}{2}\right\}\right)+\mathbb{P}\left(\left\{|\nabla f(y^m_i)|^2>\frac{\mu \textcolor{black}{(M_f-f^*)}}{2}\right\}\right)\,,
\end{aligned}
\end{equation}
where $y^m_i$ is defined in \eqref{eqn:yt}-\eqref{eqn:ymolmc} and we use $2|a-b|^2+2|b|^2\geq |a|^2$ in the second inequality.

The second term of \eqref{eqn:boundofP20} is easy to bound:
\begin{equation}\label{eqn:boundofP23}
\begin{aligned}
\mathbb{P}\left(\left\{|\nabla f(y^m_i)|^2>\mu\sqrt{N^\ast}/2\right\}\right)\leq \frac{2}{\mu \textcolor{black}{(M_f-f^*)}}\EE\left(|\nabla f(y^m_i)|^2\right)\leq \frac{2\kappa d}{\textcolor{black}{(M_f-f^*)}}\,,
\end{aligned}
\end{equation}
where we use $\EE_p|\nabla f(y)|^2\leq Ld$ according to Lemma 3 in \cite{DALALYAN20195278}. 

The first term can be bounded by 
\begin{equation}\label{eqn:boundofP21}
\begin{aligned}
&\mathbb{P}\left(\left\{|\nabla f(x^m_i)-\nabla f(y^m_i)|^2>\frac{\mu \textcolor{black}{(M_f-f^*)}}{2}\right\}\right)\leq \mathbb{P}\left(\left\{|x^m_i-y^m_i|^2>\frac{\mu \textcolor{black}{(M_f-f^*)}}{2L^2}\right\}\right)\\
\leq&\mathbb{P}\left(\left\{|x^m_i-y^m_i|>\frac{(\mu \textcolor{black}{(M_f-f^*)})^{1/2}}{\sqrt{2}L}\right\}\right)\leq \sqrt{\frac{2\kappa L}{\textcolor{black}{(M_f-f^*)}}}\EE(|x^m_i-y^m_i|)
\end{aligned}
\end{equation}
where we use $|\nabla f(x^m_i)-\nabla f(y^m_i)|\leq L|x^m_i-y^m_i|$ in the first inequality. Plugging \eqref{eqn:boundofP23} and \eqref{eqn:boundofP21} into right-side of \eqref{eqn:boundofP20}, we prove \eqref{boundofP2} by \eqref{eqn:Aconvergence}.

Finally, \eqref{boundofP3} is a direct result of \eqref{eqn:Aconvergence} and the second inequality in Lemma \ref{lem:largeN2} \eqref{eqn:iterationboundnew2}.
\end{proof}

\appendix
\section{Proof of Proposition \ref{thm:disyz}}\label{pfthm:disyz}
In this section, we prove Proposition \ref{thm:disyz}. For convenience, we ignore $i$ and define
\[
\Delta^m=z^m-y^m\,.
\]
Then it suffices to prove the smallness of $\EE|\Delta^m|$.
\begin{proof}[Proof of Proposition \ref{thm:disyz}]
we first divide $\Delta^{m+1}$ into several parts:
\begin{equation}\label{eqn:Deltam+1}
\begin{aligned}
\Delta^{m+1}=&\Delta^m+(y^{m+1}-y^{m})-(z^{m+1}-z^m)\\
=&\Delta^m+\left(-\int^{(m+1)h}_{mh}\nabla f(y(s))\rd s+\sqrt{2h}\xi_m\right)\\
&-\left(-\int^{(m+1)h}_{mh}\nabla f(z^m)\rd s+\sqrt{2h}\xi_m\right)\\
=&\Delta^m-\left(\int^{(m+1)h}_{mh}\left(\nabla f(y(s))-\nabla f(z^m)\right)\rd s\right)\\
=&\Delta^m-\left(\int^{(m+1)h}_{mh}\left(\nabla f(y(s))-\nabla f(y^m)+\nabla f(y^m)-\nabla f(z^m)\right)\rd s\right)\\
=&\Delta^m-h\left(\nabla f(y^m)-\nabla f(z^m)\right)-\int^{(m+1)h}_{mh}\left(\nabla f(y(s))-\nabla f(y^m)\right)\rd s\\
=&\Delta^m-hU^m-V^m\\
\end{aligned}\,,
\end{equation}
where 
\[
\begin{aligned}
U^m&=\nabla f(y^m)-\nabla f(z^m)\,,\\
V^m&=\int^{(m+1)h}_{mh}\left(\nabla f(y(s))-\nabla f(y^m)\right)\rd s\,.\\
\end{aligned}
\]
Now the first two terms of \eqref{eqn:Deltam+1} can be bounded by 
\begin{equation}\label{firstermofeqn:Deltam+1}
\left|\Delta^m-hU^m\right|\leq (1-\mu h)\left|\Delta^m\right|\,,
\end{equation}
where we use $f$ is $\mu$-convex.

Next, for the second term on the right-hand side of \eqref{eqn:Deltam+1}, we first bound $L^2$-norm:
\begin{align}
\nonumber
\EE\left(|V^m|^2\right)&\stackrel{\text{(I)}}{\leq} h \int^{(m+1)h}_{mh}\EE\left(\left|\nabla  f(y(s))-\nabla  f(y^m)\right|^2\right)\rd s\\
\nonumber
&\stackrel{\text{(II)}}{\leq} hL^2 \int^{(m+1)h}_{mh}\EE\left(\left|y(s)-y^m\right|^2\right)\rd s\\
\nonumber
&= hL^2 \int^{(m+1)h}_{mh}\EE\left(\left|\int^{s}-{mh} \nabla  f(y(t))\rd t+\sqrt{2}(B(s)-B(nh))\right|^2\right)\rd s\\
\nonumber
&\stackrel{\text{(III)}}{\leq} 2h^2L^2 \int^{(m+1)h}_{mh}\int^{s}-{mh} \EE\left(\left|\nabla  f(y(t))\right|^2\right)\rd t\rd s \\
\nonumber
& \quad\quad +4h^2L^2\int^{(m+1)h}_{mh} \EE|\xi^m|^2\rd s\\
\nonumber
&\stackrel{\text{(IV)}}{=} h^4L^2\EE\left(\left|\nabla  f(y^m)\right|^2\right)+4h^3L^2d\\
\label{secondtermDeltam+1}
&\stackrel{\text{(V)}}{=}h^4L^2\EE_p|\nabla f|^2+4h^3L^2\stackrel{\text{(VI)}}{\leq} h^4L^3d+4h^3L^2d\,,
\end{align}
where $\text{(II)}$ comes from $L$-Lipschitz condition, $\text{(I)}$ and $\text{(III)}$ come from the use of Young's inequality and    Jensen's inequality when we move the $|\cdot|^2$ from outside to inside of the integral, and $\text{(IV)}$ and $\text{(V)}$ hold true because $y(t)\sim p$ for all $t$. In $\text{(VI)}$ we use $\EE_p|\nabla f|^2\leq Ld$ using~\cite[Lemma~3]{DALALYAN20195278}.

Using H\"older's inequality and $h\leq \frac{1}{L}$, \eqref{secondtermDeltam+1} implies
\[
\EE\left(|V^m|\right)\leq \left(\EE\left(|V^m|^2\right)\right)^{1/2}\leq 5h^{3/2}Ld^{1/2}\,.
\]
Plugging this and \eqref{firstermofeqn:Deltam+1} into \eqref{eqn:Deltam+1}, we obtain
\[
\EE\left(\left|\Delta^{m+1}\right|\right)\leq \EE\left(\left|\Delta^m-hU^m\right|\right)+\EE\left(|V^m|\right)\leq (1-\mu h)\EE\left(\left|\Delta^m\right|\right)+5h^{3/2}Ld^{1/2}\,.
\]
Using this iteratively and $\EE|\Delta^0|=\EE|z^0-y^0|=W_{1}(q^0,p)$, we prove \eqref{eqn:thm:disyz}.
\end{proof}

\section{Other choices of ensemble gradient approximation}\label{sec:ega_FD}
The ensemble gradient approximation we present in Section~\ref{sec:ega} is of probability type, namely, we take the ensemble average of finite difference around $x^\ast$. There are other ways to find gradient approximations as well, and probably the most straightforward method is to solve a linear algebra problem formulated by the closest $d$ neighbors.

More specifically, let $\eta>0$ and $x^*\in\mathbb{R}^d$. Assume that there are $d$ points $\{x_i\}^d_{i=1}$ in the ball $B_{\eta}(x^*)$, then we have
\[
\Delta_x\cdot \nabla f(x^*)=\Delta_f+o(\eta)\,,
\]
where
\begin{equation}\label{def:Deltaxf}
\Delta_x=\left[
\begin{aligned}
&(x_1-x^*)^\top\\
&(x_2-x^*)^\top\\
&\dots\\
&(x_d-x^*)^\top\\
\end{aligned}
\right],\quad \Delta_f=\left[
\begin{aligned}
&f(x_1)-f(x^*)\\
&f(x_2)-f(x^*)\\
&\dots\\
&f(x_d)-f(x^*)\\
\end{aligned}
\right]\,.
\end{equation}
If $\Delta_x$ is full rank, then by solving the equation $\Delta_x\cdot z=\Delta_f$, we obtain an approximation of the gradient
\[
z\approx\nabla f(x^*)\,.
\]

A natural question to ask is, how likely is it to find $d$ neighbors in a small neighborhood of a given sample? To quantify such probability, we use the following lemma:
\begin{lemma}\label{lem:largeN}
Suppose $|p(x)|\leq M<\infty$ and $\{x_i\}^{N}_{i=1}$ are i.i.d. drawn from $p$ with $N>0$. Let $N=c/\eta^d$, where $c$ is a positive constant. Then we have 
\[
\limsup_{\eta\rightarrow0} \mathbb{P}\left(\#\left\{x_i\middle||x_i-x_1|<\eta,\ i=1,2,\dots,N\right\}\geq d+1\right)\leq 1-\exp\left(-cM\right)\,.
\]
\end{lemma}
This lemma can be viewed as a negative result: even with $N$ exponentially big on $d$, there is still a nontrivial chance for a sample to not have enough neighbors around for the gradient computation.

\begin{proof}[Proof of Lemma \ref{lem:largeN}]
Fixed $x_1\in\mathbb{R}^d$,
\[
\begin{aligned}
\mathbb{P}\left(|x_2-x_1|<\eta\middle|x_1\right)&=\int_{|z|<\eta} p(x_1+z)\rd z\leq \eta^dM\,.
\end{aligned}
\]
Denote $p=\mathbb{P}\left(|x_2-x_1|<\eta\middle|x_1\right)$. Because $\{x_i\}^N_{i=1}$ are independent, we have
\[
\begin{aligned}
&\mathbb{P}\left(\#\left\{x_i\middle||x_i-x_1|<\eta,\ i=1,2,\dots,N\right\}<d+1\middle|x_1\right)\\
=&\sum_{k=1}^d\mathbb{P}\left(\#\left\{x_i\middle||x_i-x_1|<\eta,\ i=1,2,\dots,N\right\}=k\middle|x_1\right)\\
=&\sum_{k=0}^{d-1} {N-1 \choose k}p^k(1-p)^{N-1-k}\geq (1-p)^{N-1}\,.
\end{aligned}
\]

Since $c=N\eta^d$, 
\[
\begin{aligned}
&\limsup_{\eta\rightarrow0}\mathbb{P}\left(\#\left\{x_i\middle||x_i-x_1|<\eta,\ i=1,2,\dots,N\right\}<d+1\middle|x_1\right)\\
\geq &\limsup_{\eta\rightarrow0}(1-p)^{N-1}\geq \limsup_{\eta\rightarrow0}\left(1- \eta^dM\right)^{\frac{c}{\eta^d}-1}\\
\geq &\exp\left(-cM\right)\,.
\end{aligned}
\]
This implies
\[
\limsup_{\eta\rightarrow0}\mathbb{P}\left(\#\left\{x_i\middle||x_i-x_1|<\eta,\ i=1,2,\dots,N\right\}<d+1\right)\geq \exp\left(-cM\right)\,.
\]
which concludes the proof.
\end{proof}
\bibliographystyle{apalike}
\bibliography{aims}

\medskip
Received xxxx 20xx; revised xxxx 20xx.
\medskip

\end{document}